\documentclass[runningheads,envcountsame]{llncs}

\usepackage{microtype,xspace}
\usepackage{cite}
\usepackage[T1]{fontenc}
\usepackage{bbm,mathtools,amsmath,amssymb,stmaryrd}
\usepackage{bbold}
\usepackage{paralist}
\usepackage{wrapfig}
\usepackage{tikz}
\usetikzlibrary{decorations}
\usetikzlibrary{automata}
\usetikzlibrary{shapes.multipart}
\usetikzlibrary{arrows}
\usepackage{mysty}

\title{Measuring Global Similarity between Texts}

\author{Uli Fahrenberg\inst1 \and Fabrizio Biondi\inst1 \and Kevin
  Corre\inst1 \and Cyrille Jegourel\inst1 \and Simon Kongsh{\o}j\inst2
  \and Axel Legay\inst1} \authorrunning{Fahrenberg, Biondi, Corre,
  Jegourel, Kongsh{\o}j, Legay}

\institute{Inria / IRISA Rennes, France \and University College of
  Northern Denmark}

\begin{document}

\maketitle

\begin{abstract}
  We propose a new similarity measure between texts which, contrary to
  the current state-of-the-art approaches, takes a global view of the
  texts to be compared.  We have implemented a tool to compute our
  textual distance and conducted experiments on several corpuses of
  texts.  The experiments show that our methods can reliably identify
  different global types of texts.
\end{abstract}

\section{Introduction}

Statistical approaches for comparing texts are used for example in
\emph{machine translation} for assessing the quality of machine
translation tools~\cite{DBLP:conf/acl/PapineniRWZ02,
  DBLP:conf/naacl/LinH03, DBLP:conf/acl/LinO04}, or in
\emph{computational linguistics} in order to establish
authorship~\cite{conf/dhcase/TomasiBCDGV13, DBLP:journals/jql/LabbeL01,
  DBLP:journals/jql/Savoy12, DBLP:journals/tois/Savoy12,
  DBLP:journals/jql/Labbe07, DBLP:journals/jql/CortelazzoNT13} or to
detect ``fake'', \ie~automatically generated, scientific
papers~\cite{DBLP:journals/scientometrics/LabbeL13,
  news/nature/Gibberish14}.

Generally speaking, these approaches consist in computing
\emph{distances}, or \emph{similarity measures}, between texts and then
using statistical methods such as, for instance, hierarchical
clustering~\cite{books/KaufmanR90} to organize the distance data and
draw conclusions.

The distances between texts which appear to be the most popular,
\eg~\cite{DBLP:journals/jql/LabbeL01, DBLP:conf/acl/PapineniRWZ02}, are
all based on measuring differences in \emph{$1$-gram frequencies}: For
each $1$-gram (token, or word) $w$ in the union of $A$ and $B$, its
absolute frequencies in both texts are calculated, \ie~$F_A( w)$ and
$F_B( w)$ are the numbers of occurrences of $w$ in $A$ and $B$,
respectively, and then the distance between $A$ and $B$ is defined to be
the sum, over all words $w$ in the union of $A$ and $B$, of the absolute
differences $| F_A( w)- F_B( w)|$, divided by the combined length of $A$
and $B$ for normalization.  When the texts $A$ and $B$ have different
length, some adjustments are needed; also, some
algorithms~\cite{DBLP:conf/acl/PapineniRWZ02, DBLP:conf/naacl/LinH03}
take into account also $2$-, $3$- and $4$-grams.

These distances are thus based on a \emph{local} model of the texts:
they measure differences of the multisets of $n$-grams for $n$ between
$1$ and maximally $4$.  Borrowing techniques from economics and
theoretical computer science, we will propose below a new distance which
instead builds on the \emph{global} structure of the texts.  It
simultaneously measures differences in occurrences of $n$-grams for all
$n$ and uses a discounting parameter to balance the influence of long
$n$-grams versus short $n$-grams.

Following the example of~\cite{DBLP:journals/scientometrics/LabbeL13},
we then use our distance to automatically identify ``fake'' scientific
papers.  These are ``papers'' which are automatically generated by some
piece of software and are hence devoid of any meaning, but which, at
first sight, have the \emph{appearance} of a genuine scientific paper.

We can show that using our distance and hierarchical clustering, we are
able to automatically identify such fake papers, also papers generated
by other methods than the ones considered
in~\cite{DBLP:journals/scientometrics/LabbeL13}, and that, importantly,
some parts of the analysis become more reliable the higher the
discounting factor.  We conclude that measuring \emph{global}
differences between texts, as per our method, can be a more reliable way
than the current state-of-the-art methods to automatically identify fake
scientific papers.  We believe that this also has applications in other
areas such machine translation or computational linguistics.

% Also, we shall argue in the conclusion that our statistical distance
% measuring method may easily be combined with other, more logical or
% structural approaches.  Indeed, we can just as easily measure
% distances between \emph{trees} as between texts, which opens up for
% applications for example to derivation trees and other structural
% methods.

\section{Inter-textual Distances}

For the purpose of this paper, a \emph{text} $A$ is a sequence $A=( a_1,
a_2,\dots, a_{ N_A})$ of words.  The number $N_A$ is called the
\emph{length} of $A$.  As a vehicle for showing idealized properties, we
may sometimes also speak of \emph{infinite} texts, but most commonly,
texts are finite and their length is a natural number.  Note that we pay
no attention to punctuation, structure such as headings or footnotes, or
non-textual parts such as images.

\subsection{1-gram distance}

Before introducing our global distance, we quickly recall the definition
of standard $1$-gram distance, which stands out as a rather popular
distance in computational linguistics and other
areas~\cite{DBLP:journals/scientometrics/LabbeL13,
  DBLP:journals/jql/Savoy12, DBLP:journals/tois/Savoy12,
  DBLP:journals/jql/Labbe07, DBLP:journals/jql/CortelazzoNT13,
  conf/dhcase/TomasiBCDGV13, DBLP:journals/jql/LabbeL01,
  DBLP:journals/lalc/LabbeL06}.

For a text $A=( a_1, a_2,\dots, a_{ N_A})$ and a word $w$, the natural
number $F_A( w)=|\{ i\mid a_i= w\}|$ is called the \emph{absolute
  frequency} of $w$ in $A$: the number of times (which may be $0$) that
$w$ appears in $A$.  We say that $w$ is \emph{contained} in $A$ and
write $w\in A$ if $F_A( w)\ge 1$.

For texts $A=( a_1, a_2,\dots, a_{ N_A})$, $B=( b_1, b_2,\dots, b_{
  N_B})$, we write $A\circ B=( a_1,\dots, a_{ N_A}, b_1,\dots, b_{
  N_B})$ for their \emph{concatenation}.  With this in place, the
\emph{1-gram distance} between texts $A$ and $B$ \emph{of equal length}
is defined to be
\begin{equation*}
  \dlabbe( A, B)= \frac{ \sum_{ w\in A\circ B}| F_A( w)- F_B( w)|}{ N_A+
    N_B}\,,
\end{equation*}
where $| F_A( w)- F_B( w)|$ denotes the absolute difference between the
absolute frequencies $F_A( w)$ and $F_B( w)$.

For texts $A$ and $B$ which are not of equal length, scaling is used:
for $N_A< N_B$, one lets
\begin{equation*}
  \dlabbe( A, B)= \frac{ \sum_{ w\in A\circ B}| F_A( w)- F_B( w) \frac{
      N_A}{ N_B}|}{ 2 N_A}\,.
\end{equation*}

By counting occurrences of $n$-grams instead of 1-grams, similar
$n$-gram distances may be defined for all $n\ge 1$.  The BLEU
distance~\cite{DBLP:conf/acl/PapineniRWZ02} for example, popular for
evaluation of machine translation, computes $n$-gram distance for $n$
between $1$ and $4$.

\subsection{Global distance}

To compute our global inter-textual distance, we do not compare word
frequencies as above, but \emph{match $n$-grams} in the two texts
\emph{approximately}.  Let $A=( a_1, a_2,\dots, a_{ N_A})$ and $B=( b_1,
b_2,\dots, b_{ N_B})$ be two texts, where we make no assertion about
whether $N_A< N_B$, $N_A= N_B$ or $N_A> N_B$.  Define an indicator
function $\delta_{ i, j}$, for $i\in\{ 1,\dots, N_A\}$, $j\in\{ 1,\dots,
N_B\}$, by
\begin{equation}
  \label{eg:delta}
  \delta_{ i, j}=
  \begin{cases}
    0 &\text{if } a_i= b_j\,,\\
    1 &\text{otherwise}
  \end{cases}
\end{equation}
(this is the \emph{Kronecker delta} for the two sequences $A$ and $B$).
The symbol $\delta_{ i, j}$ indicates whether the $i$-th word $a_i$ in
$A$ is the same as the $j$-th word $b_j$ in $B$.  For ease of notation,
we extend $\delta_{ i, j}$ to indices above $i, j$, by declaring
$\delta_{ i, j}= 1$ if $i> N_A$ or $j> N_B$.

Let $\lambda\in \Real$, with $0\le \lambda< 1$, be a \emph{discounting
  factor}.  Intuitively, $\lambda$ indicates how much weight we give to
the length of $n$-grams when matching texts: for $\lambda= 0$, we match
$1$-grams only (see also Theorem~\ref{th:labbe} below), and the higher
$\lambda$, the longer the $n$-grams we wish to match.  Discounting is a
technique commonly applied for example in economics, when gauging the
long-term effects of economic decisions.  Here we remove it from its
time-based context and apply it to \emph{$n$-gram length} instead: We
define the \emph{position match} from any position index pair $( i, j)$
in the texts by
\begin{equation}
  \label{eq:phrasedist}
  \begin{aligned}
    \dphrase( i, j, \lambda) &=  \delta_{ i, j}+
    \lambda \delta_{ i+ 1, j+ 1}+ \lambda^2 \delta_{ i+ 2, j+ 2}+\dotsb \\
    &= \sum_{ k= 0}^\infty \lambda^k \delta_{ i+ k, j+ k}\,.
  \end{aligned}
\end{equation}

This measures how much the texts $A$ and $B$ ``look alike'' when
starting with the tokens $a_i$ in $A$ and $b_j$ in $B$.  Note that it
takes values between $0$ (if $a_i$ and $b_j$ are the starting points for
two equal infinite sequences of tokens) and $\frac1{ 1- \lambda}$.
Intuitively, the more two token sequences are alike, and the later they
become different, the smaller their distance.
Table~\ref{ta:phrase_distance} shows a few examples of position match
calculations.

\begin{table}[tp]
  \centering
  \caption{%
    \label{ta:phrase_distance}
    Position matches, starting from index pair $(1, 1)$ and scaled by
    $1- \lambda$, of different example texts, for general discounting
    factor $\lambda$ and for $\lambda= .8$.  Note that the last two
    example texts are infinite.
  }
  \setlength{\tabcolsep}{3pt}
  \begin{tabular}{c|c|l|l}
    Text $A$ & Text $B$ & $( 1- \lambda)\dphrase( 1, 1, \lambda)$ &
    $\lambda= .8$ \\\hline
    ``man'' & ``dog'' & $1$ & $1$ \\
    ``dog'' & ``dog'' & $\lambda$ & $.8$ \\
    ``man bites dog'' & ``man bites dog'' & $\lambda^3$ & $.51$ \\
    ``man bites dog'' & ``dog bites man'' & $1- \lambda+ \lambda^2$ &
    $.84$ \\
    ``the quick brown fox \qquad\qquad& ``the quick white
    fox\qquad\qquad\quad & $\lambda^2- \lambda^3+ \lambda^4- \lambda^6$ \\
    jumps over the lazy dog'' &  crawls under the high
    dog'' & \quad\ $\mathop+ \lambda^7- \lambda^8+ \lambda^9$ & $.45$ \\
    ``me me me me...'' & ``me me me me...'' & $0$ & $0$
  \end{tabular}
\end{table}

\begin{table}[bp]
  \centering
  \caption{%
    \label{ta:matrix}
    Position match matrix example, with  discounting factor $\lambda=
    .8$.
  }
  \setlength{\tabcolsep}{3pt}
  \begin{tabular}{c|cccccccc}
    & the & quick & fox & jumps & over & the & lazy & dog \\\hline
    the & 0.67 & 1.00 & 1.00 & 1.00 & 1.00 & 0.64 & 1.00 & 1.00 \\
    lazy & 1.00 & 0.84 & 1.00 & 1.00 & 1.00 & 1.00 & 0.80 & 1.00 \\
    fox & 1.00 & 1.00 & 0.80 & 1.00 & 1.00 & 1.00 & 1.00 & 1.00 
  \end{tabular}
\end{table}

This gives us an $N_A$-by-$N_B$ matrix $\Dphrase( \lambda)$ of position
matches; see Table~\ref{ta:matrix} for an example.  We now need to
consolidate this matrix into \emph{one} global distance value between
$A$ and $B$.  Intuitively, we do this by averaging over position
matches: for each position $a_i$ in $A$, we find the position $b_j$ in
$B$ which best matches $a_i$, \ie~for which $\dphrase( i, j, \lambda)$
is minimal, and then we average over these matchings.

Formally, this can be stated as an \emph{assignment problem}: Assuming
for now that $N_A= N_B$, we want to find a matching of indices $i$ to
indices $j$ which minimizes the sum of the involved $\dphrase( i, j)$.
Denoting by $S_{ N_A}$ the set of all permutations of indices $\{
1,\dotsc, N_A\}$ (the \emph{symmetric group} on $N_A$ elements), we hence
define
\begin{equation*}
  d_2( A, B, \lambda)=( 1- \lambda)\frac1{ N_A} \min_{ \phi\in S_{ N_A}}
  \sum_{ i= 1}^{ N_A} \dphrase( i, \phi( i), \lambda)\,.
\end{equation*}

This is a conservative extension of $1$-gram distance, in the sense that
for discounting factor $\lambda= 0$ we end up computing $\dlabbe$:

\begin{theorem}
  \label{th:labbe}
  For all texts $A$, $B$ with equal lengths, $d_2( A, B, 0)= \dlabbe( A,
  B)$.
\end{theorem}

\begin{proof}
  For $\lambda= 0$, the entries in the phrase distance matrix are
  $\dphrase( i, j, 0)= \delta_{ i, j}$.  Hence a \emph{perfect match} in
  $\Dphrase$, with $\sum_{ i= 1}^{ N_A} \dphrase( i, \phi( i), 0)= 0$,
  matches each word in $A$ with an equal word in $B$ and vice versa.
  This is possible if, and only if, $F_A( w)= F_B( w)$ for each word
  $w$.  Hence $d_2( A, B, 0)= 0$ iff $\dlabbe( A, B)= 0$.  The proof of
  the general case is in appendix. \qed
\end{proof}

% \begin{proposition}
%   $d_1( A, A, \lambda) =0$,\uli{Not sure we need this here.  Maybe in
%   appendix.}  $d_1( A, B, \lambda)= d_1( B, A, \lambda)$, $d_1( A, B,
%   \lambda)+ d_1( B, C, \lambda)> d_1( A, C, \lambda)$, and for
%   $\lambda> 0$ and $A$, $B$ finite texts, $d_1( A, B, \lambda)= 0$
%   implies $A= B$.
% \end{proposition}

There are, however, some problems with the way we have defined $d_2$.
For the first, the assignment problem is computationally rather
expensive: the best know algorithm (the \emph{Hungarian
  algorithm}~\cite{journals/nrl/Kuhn55}) runs in time cubic in the size
of the matrix, which when comparing large texts may result in
prohibitively long running times.  Secondly, and more important, it is
unclear how to extend this definition to texts which are not of equal
length, \ie~for which $N_A\ne N_B$.  (The scaling approach does not work
here.)

Hence we propose a different definition which has shown to work well in
practice, where we abandon the idea that we want to match phrases
\emph{uniquely}.  In the definition below, we simply match every phrase
in $A$ with its best equivalent in $B$, and we do not take care whether
we match two different phrases in $A$ with the same phrase in $B$.
Hence,
\begin{equation*}
  d_3( A, B, \lambda)= \frac1{ N_A} \sum_{ i= 1}^{ N_A} \min_{ j= 1,\dotsc, N_B}
  \dphrase( i, j, \lambda)\,.
\end{equation*}

Note that $d_3( A, B, \lambda)\le d_2( A, B, \lambda)$, and that
contrary to $d_2$, $d_3$ is \emph{not symmetric}.  We can fix this by
taking as our final distance measure the symmetrization of $d_3$:
\begin{equation*}
  d_4( A, B, \lambda)= \max( d_3( A, B, \lambda), d_3( B, A, \lambda))\,.
\end{equation*}

\section{Implementation}

We have written a \texttt C~program and some \texttt{bash} helper
scripts which implement the computations above.  All our software is
available at~\url{http://textdist.gforge.inria.fr/}.

The \texttt C~program, \texttt{textdist.c}, takes as input a list of
\texttt{txt}-files $A_1, A_2,\dotsc, A_k$ and a discounting factor
$\lambda$ and outputs $d_4( A_i, A_j, \lambda)$ for all pairs $i, j=
1,\dotsc, k$.  With the current implementation, the \texttt{txt}-files
can be up to 15,000 words long, which is more than enough for all texts
we have encountered.  On a standard 3-year-old business laptop
(Intel\textregistered\ Core\texttrademark\ i5 at 2.53GHz$\times$4),
computation of $d_4$ for takes less than one second for each pair of
texts.

We preprocess texts to convert them to \texttt{txt}-format and remove
non-word tokens.  The \texttt{bash}-script \texttt{preprocess-pdf.sh}
takes as input a \texttt{pdf}-file and converts it to a text file, using
the \texttt{poppler} library's \texttt{pdftotext} tool.  Afterwards,
\texttt{sed} and \texttt{grep} are used to convert whitespace to
newlines and remove excessive whitespace; we also remove all ``words''
which contain non-letters and only keep words of at least two letters.

The \texttt{bash}-script \texttt{compareall.sh} is used to compute
mutual distances for a corpus of texts.  Using \texttt{textdist.c} and
taking $\lambda$ as input, it computes $d_4( A, B, \lambda)$ for all
texts (\texttt{txt}-files) $A$, $B$ in a given directory and outputs
these as a matrix.  We then use \texttt{R} and \texttt{gnuplot} for
statistical analysis and visualization.

We would like to remark that all of the above-mentioned tools are free
or open-source software and available without charge.  One often forgets
how much science has come to rely on this free-software infrastructure.

\section{Experiments}

We have conducted two experiments using our software.  The data sets on
which we have based these experiments are available on request.

\subsection{Types of texts used}

We have run our experiments on papers in computer science, both genuine
papers and automatically generated ``fake'' papers.  As to the genuine
papers, for the first experiment, we have used 42 such papers from
within theoretical computer science, 22 from the proceedings of the
FORMATS 2011 conference~\cite{DBLP:conf/formats/2011} and 20 others
which we happened to have around.  For the second experiment, we
collected 100 papers from \url{arxiv.org}, by searching their Computer
Science repository for authors named ``Smith'' (\texttt{arxiv.org}
strives to prevent bulk paper collection), of which we had to remove
three due to excessive length (one ``status report'' of more than 40,000
words, one PhD thesis of more than 30,000 words, and one ``road map'' of
more than 20,000 words).

We have employed three methods to collect automatically generated
``papers''.  For the first experiment, we downloaded four fake
publications by ``Ike Antkare''.  These are out of a set of 100 papers
by the same ``author'' which have been generated, using the SCIgen paper
generator, for another experiment~\cite{journals/issi/labbe10}.  For the
purpose of this other experiment, these papers all have the same
bibliography, each of which references the other 99 papers; hence not to
skew our results (and like was done
in~\cite{DBLP:journals/scientometrics/LabbeL13}), we have stripped their
bibliography.

SCIgen\footnote{\url{http://pdos.csail.mit.edu/scigen/}} is an automatic
generator of computer science papers developed in 2005 for the purpose
of exposing ``fake'' conferences and journals (by submitting generated
papers to such venues and getting them accepted).  It uses an elaborate
grammar to generate random text which is devoid of any meaning, but
which to the untrained (or inattentive) eye looks entirely legitimate,
complete with abstract, introduction, figures and bibliography.  For the
first experiment, we have supplemented our corpus with four SCIgen
papers which we generated on their website.  For the second experiment,
we modified SCIgen so that we could control the length of generated
papers and then generated 50 papers.

For the second experiment, we have also employed another paper generator
which works using a simple Markov chain model.  This program,
\texttt{automogensen}\footnote{\url{http://www.kongshoj.net/automogensen/}},
was originally written to expose the lack of meaning of many of a
certain Danish political commentator's writings, the challenge being to
distinguish genuine \textit{Mogensen} texts from ``fake''
\texttt{automogensen} texts.  For our purposes, we have modified
\texttt{automogensen} to be able to control the length of its output and
fed it with a 248,000-word corpus of structured computer science text
(created by concatenating all 42 genuine papers from the first
experiment), but otherwise, its functionality is rather simple: It
randomly selects a 3-word starting phrase from the corpus and then,
recursively, selects a new word from the corpus based on the last three
words in its output and the distribution of successor words of this
three-word phrase in the corpus.

\subsection{First experiment}

\begin{figure}[tbp]
  \hspace*{-2em}
  \includegraphics[width=.55\linewidth, trim=30 70 50 70, clip, angle=-90]{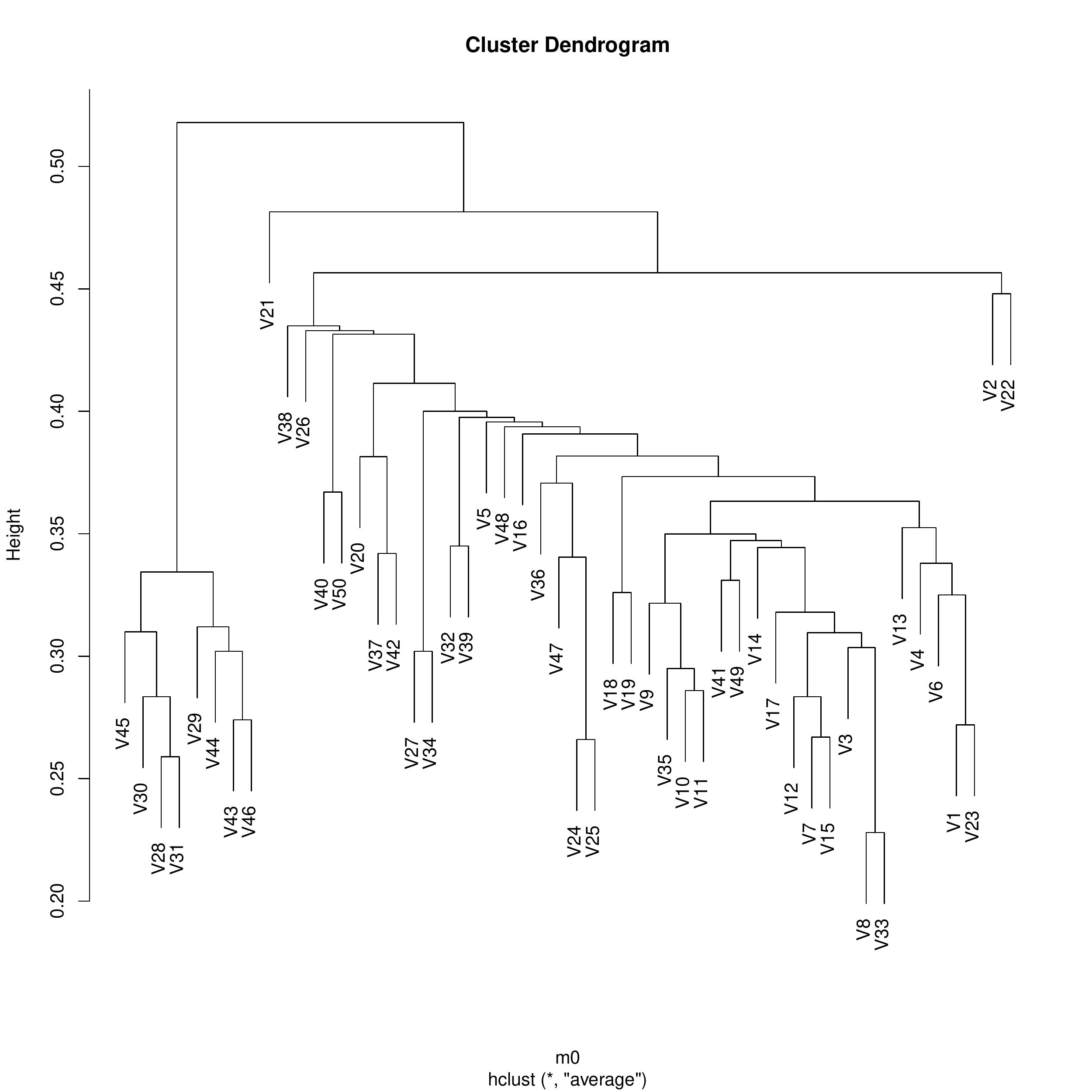}
  \hfill
  \includegraphics[width=.55\linewidth, trim=30 70 50 70, clip, angle=-90]{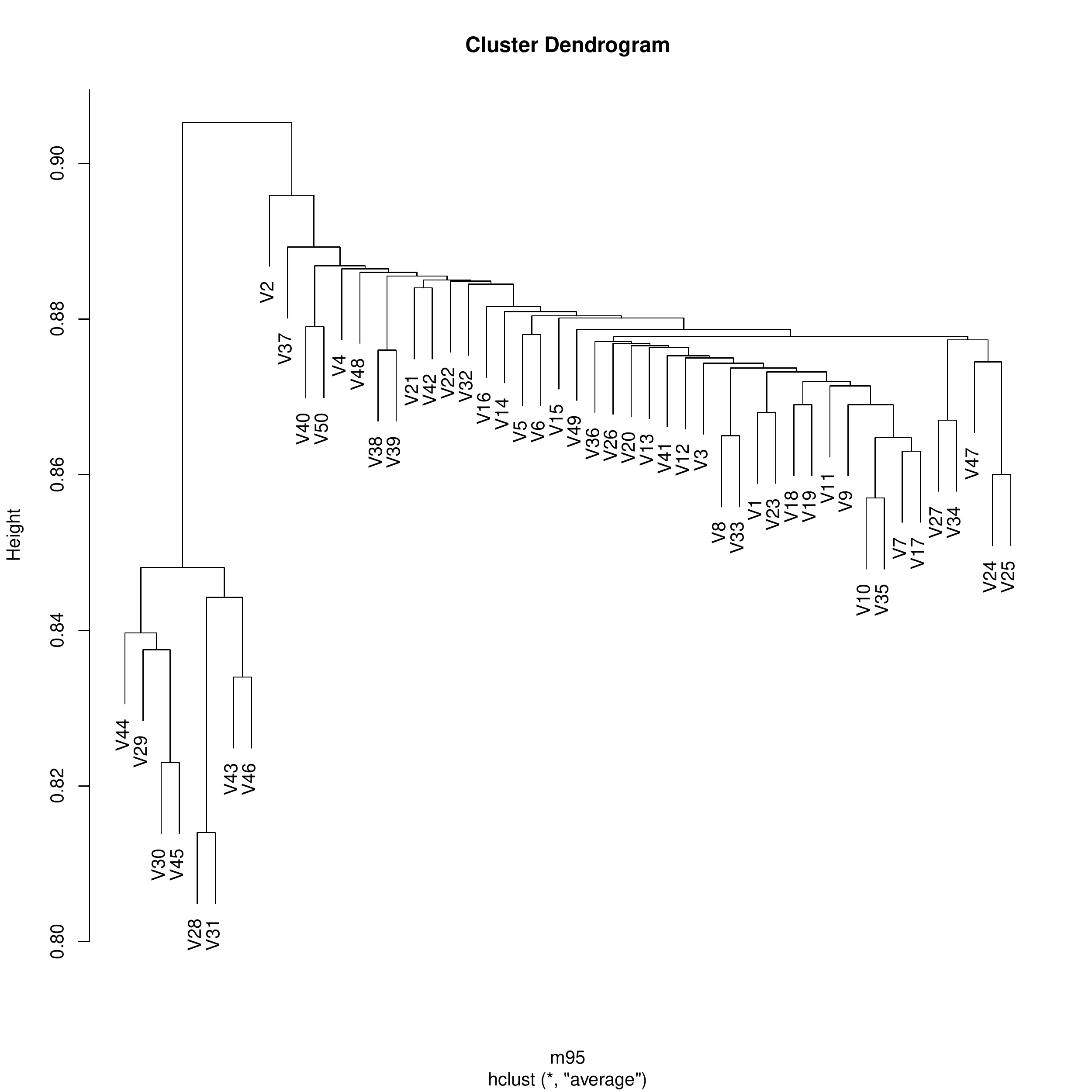}
  \hspace*{-2em}
  \caption{%
    \label{fi:dendro-1-avg.m}
    Dendrograms for Experiment 1, using average clustering, for
    discounting factors $0$ (left) and $.95$ (right), respectively.
    Fake papers are numbered 28-31 (Antkare) and 43-46 (SCIgen), the
    others are genuine.}
\end{figure}

The first experiment was conducted on 42 genuine papers of lengths
between 3,000 and 11,000 words and 8 fake papers of lengths between 1500
and 2200 words.  Figure~\ref{fi:dendro-1-avg.m} shows two dendrograms
with average clustering created from the collected distances; more
dendrograms are available in appendix.  The left dendrogram was computed
for discounting factor $\lambda= 0$, \ie~word matching only.  One
clearly sees the fake papers grouped together in the top cluster and the
genuine papers in cluster below.  In the right dendrogram, with very
high discounting ($\lambda= .95$), this distinction is much more clear;
here, the fake cluster is created (at height $.85$) while all the
genuine papers are still separate.  The dendrograms in
Fig.~\ref{fi:dendro-1-ward.m}, created using Ward clustering, clearly
show that one should distinguish the data into \emph{two} clusters, one
which turns out to be composed only of fake papers, the other only of
genuine papers.

\begin{figure}[tbp]
  \hspace*{-2em}
  \includegraphics[width=.55\linewidth, trim=30 70 50 50, clip, angle=-90]{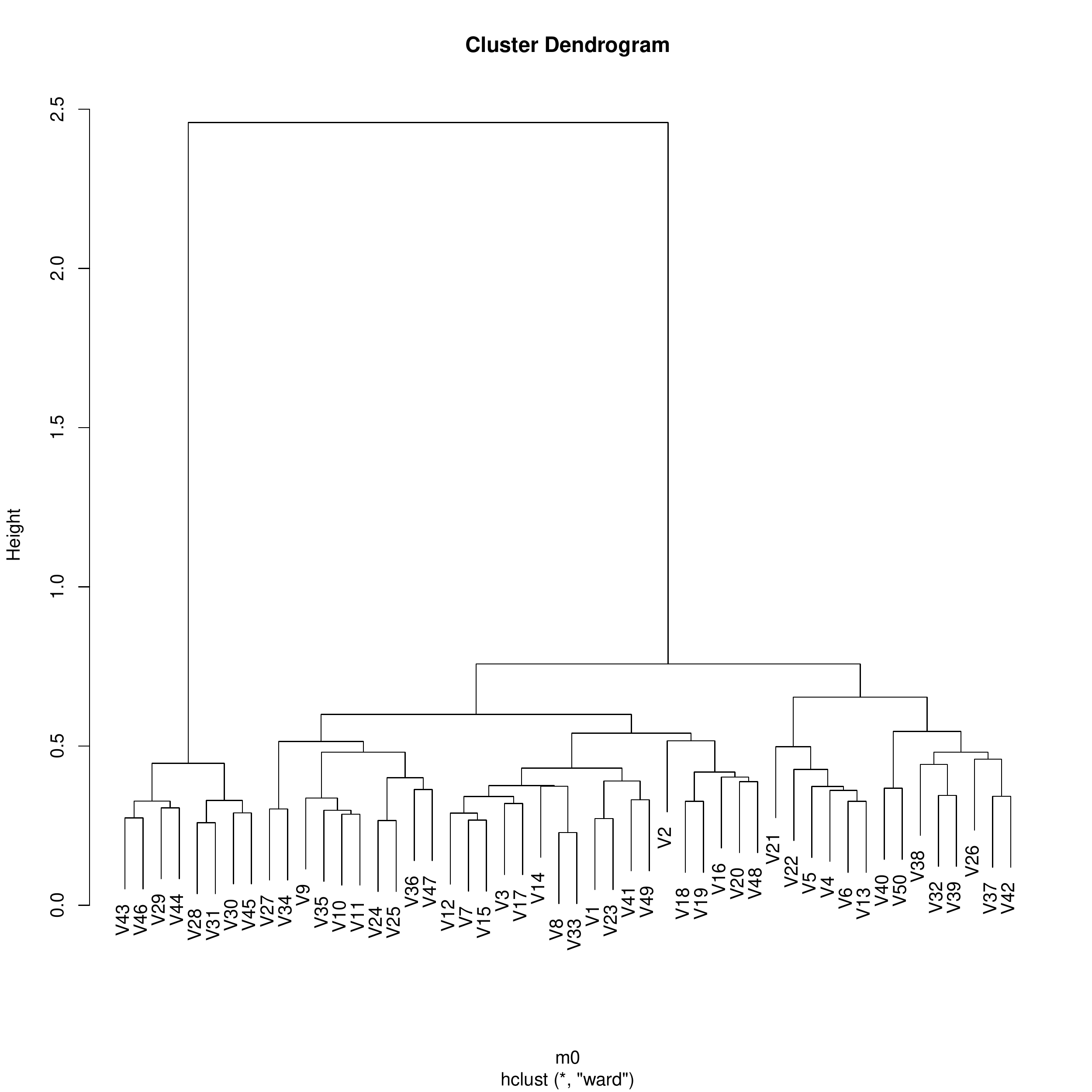}
  \hfill
  \includegraphics[width=.55\linewidth, trim=30 70 50 70, clip, angle=-90]{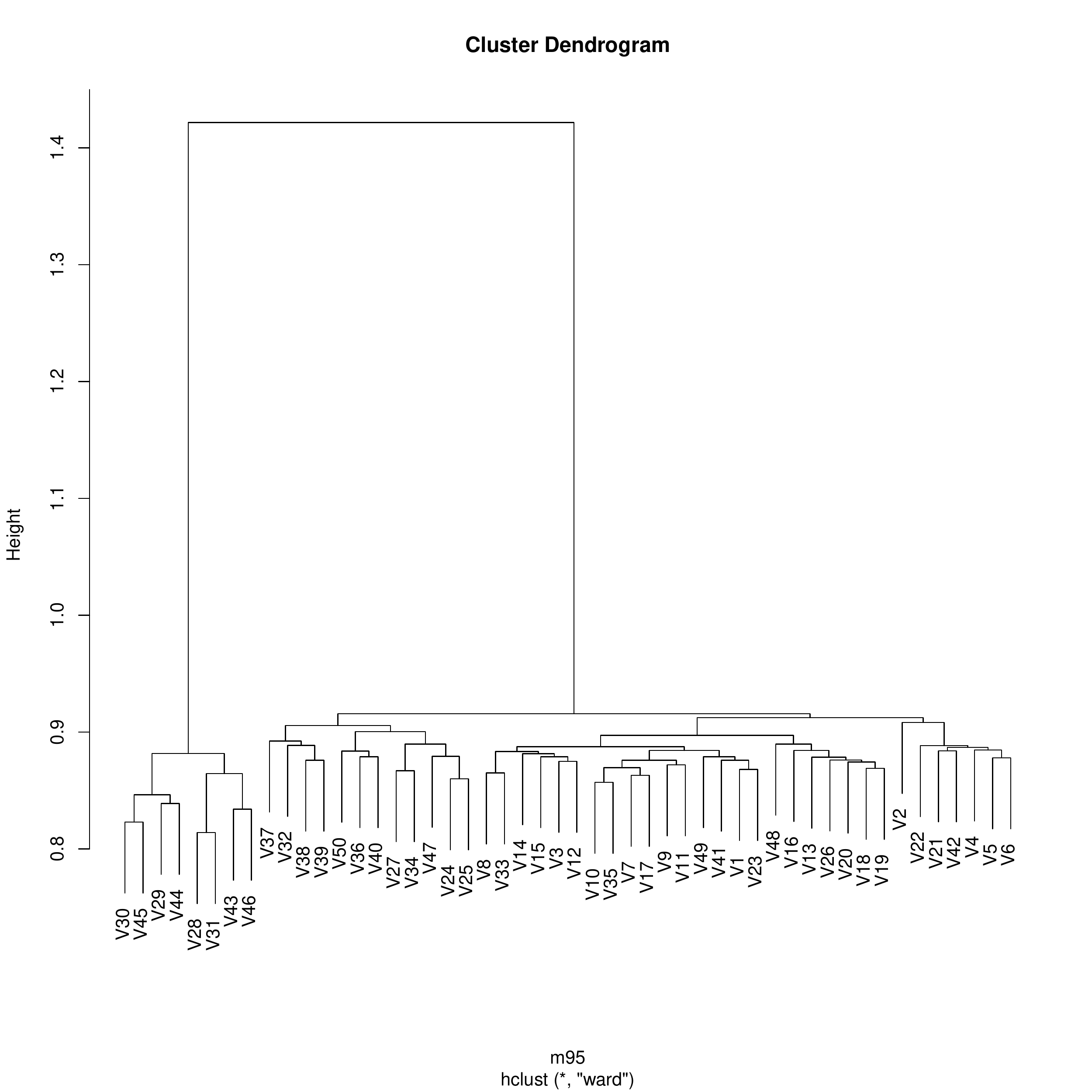}
  \hspace*{-2em}
  \caption{%
    \label{fi:dendro-1-ward.m}
    Dendrograms for Experiment 1, using Ward clustering, for discounting
    factors $0$ (left) and $.95$ (right), respectively.  Fake papers are
    numbered 28-31 (Antkare) and 43-46 (SCIgen), the others are
    genuine.}
\end{figure}

We want to call attention to two other interesting observations which
can be made from the dendrograms in Fig.~\ref{fi:dendro-1-avg.m}.
First, papers~2, 21 and 22 seem to stick out from the other genuine
papers.  While all other genuine papers are technical papers from within
theoretical computer science, these three are not.
Paper~2~\cite{DBLP:conf/formats/Haverkort11} is a non-technical position
paper, and papers~21~\cite{DBLP:conf/formats/SankaranarayananHL11}
and~22~\cite{DBLP:conf/formats/KharmehEM11} are about applications in
medicine and communication.  Note that the $\lambda= .95$ dendrogram
more clearly distinguishes the position
paper~\cite{DBLP:conf/formats/Haverkort11} from the others.

Another interesting observation concerns
papers~8~\cite{DBLP:conf/formats/BassetA11} and~33~\cite{hal/AsarinD09}.
These papers share an author (E.~Asarin) and are within the same
specialized area (topological properties of timed automata), but
published two years apart.  When measuring only word distance, \ie~with
$\lambda= 0$, these papers have the absolutely lowest distance, $.23$,
even below any of the fake papers' mutual distances, but increasing the
discounting factor increases their distance much faster than any of the
fake papers' mutual distances.  At $\lambda= .95$, their distance is
$.87$, above any of the fake papers' mutual distances.  A conclusion can
be that these two papers may have \emph{word} similarity, but they are
distinct in their \emph{phrasing}.

\begin{table}[tbp]
  \centering
  \caption{%
    \label{ta:disc-vs-dist}
    Minimal and maximal distances between different types of papers
    depending on the discounting factor.}
  \begin{tabular}{c|cccccccccccc}
    type & discounting & 0 & .1 & .2 & .3 & .4 & .5 & .6 & .7 & .8 & .9
    & .95 \\\hline 
    genuine / genuine & min & .23 & .26 & .30 & .35 & .40 & .45 & .52 &
    .59 & .68 & .79 & .86 
    \\
    & max & .55 & .56 & .57 & .59 & .61 & .64 & .67 & .72 & .78 & .85 &
    .90 \\\hline
    fake / fake & min & .26 & .28 & .31 & .35 & .39 & .43 & .49 & .55 &
    .63 & .73 & .81 \\
    & max & .38 & .40 & .43 & .46 & .49 & .53 & .58 & .64 & .71 & .80 &
    .86 \\\hline
    fake / genuine & min & .44 & .46 & .49 & .52 & .55 & .59 & .64 & .70
    & .76 & .84 & .89 \\
    & max & .58 & .60 & .62 & .64 & .66 & .68 & .72 & .76 & .80 & .87 & .92
  \end{tabular}
\end{table}

Finally, we show in Table~\ref{ta:disc-vs-dist} (see also
Fig.~\ref{fi:disc-vs-dist-1} in the appendix for a visualization) how
the mutual distances between the 50 papers evolve depending on the
discounting factor.  One can see that at $\lambda= 0$, the three types
of mutual distances are overlapping, whereas at $\lambda= .95$, they are
almost separated into three bands: .81-.86 for fake papers, .86-.90 for
genuine papers, and .89-.92 for comparing genuine with fake papers.

Altogether, we conclude from the first experiment that our inter-textual
distance can achieve a safe separation between genuine and fake papers
in our corpus, and that the separation is stronger for higher
discounting factors.

\subsection{Second experiment}

The second experiment was conducted on 97 papers from
\texttt{arxiv.org}, 50 fake papers generated by a modified SCIgen
program, and 50 fake papers generated by \texttt{automogensen}.  The
\texttt{arxiv} papers were between 1400 and 15,000 words long, the
SCIgen papers between 2700 and 12,000 words, and the
\texttt{automogensen} papers between 4,000 and 10,000 words.  The
distances were computed for discounting factors $0$, $.4$, $.8$ and
$.95$; with our software, computations took about four hours for each
discounting factor.

We show the dendrograms using average clustering in
Figs.~\ref{fi:dendro-2-avg.0} to~\ref{fi:dendro-2-avg.95} in the
appendix; they appear somewhat inconclusive.  One clearly notices the
SCIgen and \texttt{automogensen} parts of the corpus, but the
\texttt{arxiv} papers have wildly varying distances and disturb the
dendrogram.  One interesting observation is that with discounting
factor~$0$, the \texttt{automogensen} papers have small mutual distances
compared to the \texttt{arxiv} corpus, comparable to the SCIgen papers'
mutual distances, whereas with high discounting ($.95$), the
\texttt{automogensen} papers' mutual distances look more like the
\texttt{arxiv} papers'.  Note that the difficulties in clustering appear
also with discounting factor $0$, hence also when only matching words.

\begin{figure}[t]
  \centering
  \includegraphics[width=1\linewidth, trim=30 70 50 70, clip, angle=-90]{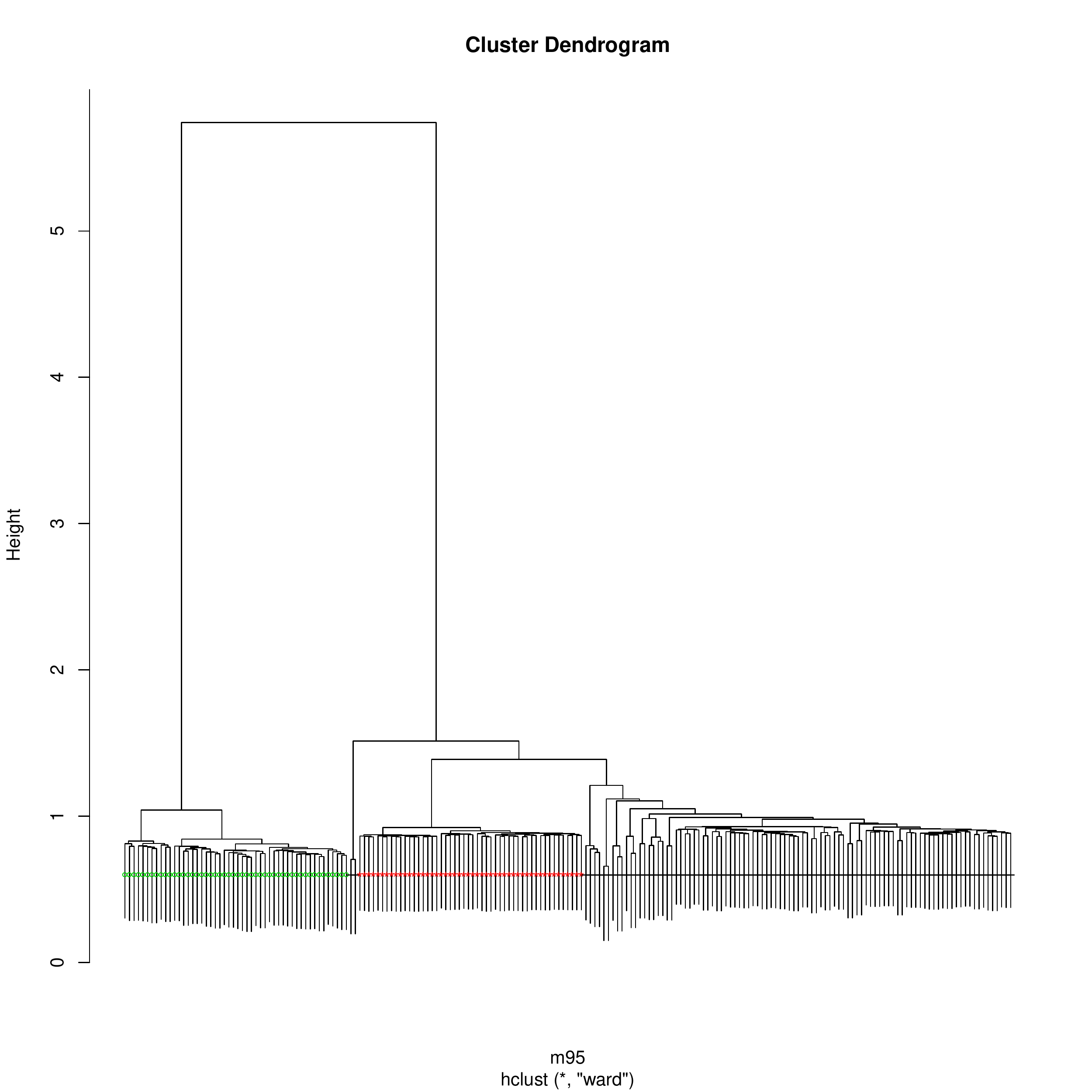}
  \caption{%
    \label{fi:dendro-2-ward.95}
    Dendrogram for Experiment 2, using Ward clustering, for discounting
    factor $.95$.  Black dots mark \texttt{arxiv} papers, green marks
    SCIgen papers, and \texttt{automogensen} papers are marked red.}
\end{figure}

The dendrograms using Ward clustering, however, do show a clear
distinction between the three types of papers.  We can only show one of
them here, for $\lambda= .95$ in Fig.~\ref{fi:dendro-2-ward.95}; the
rest are available in appendix.  One clearly sees the SCIgen cluster
(top) separated from all other papers, and then the
\texttt{automogensen} cluster (middle) separated from the \texttt{arxiv}
cluster.  

There is, though, one anomaly: two \texttt{arxiv} papers have been
``wrongly'' grouped into their own cluster (between the SCIgen and the
\texttt{automogensen} clusters).  Looking at these papers, we noticed
that here our pdf-to-text conversion had gone wrong: the papers' text
was all garbled, consisting only of ``AOUOO OO AOO EU OO OU AO'' etc.
The dendrograms rightly identify these two papers in their own cluster;
in the dendrograms using average clustering, this garbled cluster
consistently has distance $1$ to the other clusters.

We also notice in the dendrogram with average clustering and discounting
factor $.95$ (Fig.~\ref{fi:dendro-2-avg.95} in the appendix) that some
of the \texttt{arxiv} papers with small mutual distances have the same
authors and are within the same subject.  This applies
to~\cite{DBLP:journals/corr/SmithKSBP13v1}
vs.~\cite{DBLP:journals/corr/abs-1303-5613v1} and
to~\cite{DBLP:journals/corr/abs-1202-1307v2}
vs.~\cite{DBLP:journals/corr/abs-1107-0062v1}.  These similarities
appear much more clearly in the $\lambda= .95$ dendrogram than in the
ones with lower discounting factor.

As a conclusion from this experiment, we can say that whereas average
clustering had some difficulties in distinguishing between fake and
\texttt{arxiv} papers, Ward clustering did not have any problems.  The
only effect of the discounting factor we could see was in identifying
similar \texttt{arxiv} papers.
We believe that one reason for the inconclusiveness of the dendrograms
with average clustering is the huge variety of the \texttt{arxiv}
corpus.  Whereas the genuine corpus of the first experiment included
only papers from the verification sub-field of theoretical computer
science, the \texttt{arxiv} corpus is comprised of papers from a diverse
selection of research areas within computer science, including robotics,
network detection, computational geometry, constraint programming,
numerical simulation and many others.  Hence, the intra-corpus variation
in the \texttt{arxiv} corpus hides the inter-corpus variations.

\section{Conclusion and Further Work}

We believe we have collected enough evidence that our global
inter-textual distance provides an interesting alternative, or
supplement, to the standard 1-gram distance.  In our experiments, we
have seen that measuring inter-textual distance with high discounting
factor enables us to better differentiate between similar and dissimilar
texts.  More experiments will be needed to identify areas where our
global matching provides advantages over pure 1-gram matching.

With regard to identifying fake scientific papers, we remark that,
according to~\cite{DBLP:journals/scientometrics/LabbeL13}, ``[u]sing
[the 1-gram distance] to detect SCIgen papers relies on the fact that
[...] the SCIgen vocabulary remains quite poor''.  Springer has recently
announced~\cite{news/springer/SCIgen2update} that they will integrate
``[a]n automatic SCIgen detection system [...] in [their] submission
check system'', but they also notice that the ``intention [of fake
papers' authors] seems to have been to increase their publication
numbers and [...] their standing in their respective disciplines and at
their institutions''; of course, auto-detecting SCIgen papers does not
change these motivations.  It is thus reasonable to expect that
generators of fake papers will get better, so that also better tools
will be needed to detect them.  We propose that our phrase-based
distance may be such a tool.

There is room for much improvement in our distance definition.  For
once, we perform no tagging of words which could identify different
spellings or inflections of the same word.  This could easily be
achieved by, using for example the Wordnet
database\footnote{\url{http://wordnet.princeton.edu/}}, replacing our
binary distance between words in Eq.~\eqref{eg:delta} with a
quantitative measure of \emph{word similarity}.  For the second, we take
no consideration of \emph{omitted} words in a phrase; our position match
calculation in Eq.~\eqref{eq:phrasedist} cannot see when two phrases
become one-off like in ``the quick brown fox jumps...''  vs.~``the brown
fox jumps...''.

Our inter-textual distance is inspired by our work
in~\cite{journal/tcs/FahrenbergL13a, DBLP:conf/aplas/FahrenbergL13,
  DBLP:conf/fsttcs/FahrenbergLT11} and other papers, where we define
distances between arbitrary \emph{transition systems}.  Now a text is a
very simple transition system, but so is a text with ``one-off jumps''
like the one above.  Similarly, we can incorporate \emph{swapping} of
words into our distance, so that we would be computing a kind of
\emph{discounted Damerau-Levenshtein
  distance}~\cite{DBLP:journals/cacm/Damerau64} (related approaches,
generally without discounting, are used for \emph{sequence alignment} in
bioinformatics~\cite{journals/jmb/NeedlemanW70,
  journals/jmb/SmithW81}).  We have integrated this approach in an
experimental version of our \texttt{textdist} tool.

\bibliographystyle{myabbrv}
\bibliography{mybib}

\begin{thebibliography}{10}

\bibitem{hal/AsarinD09}
E.~Asarin and A.~Degorre.
\newblock Volume and entropy of regular timed languages.
\newblock {\em hal}, 2009.
\newblock \url{http://hal.archives-ouvertes.fr/hal-00369812}.

\bibitem{DBLP:conf/formats/BassetA11}
N.~Basset and E.~Asarin.
\newblock Thin and thick timed regular languages.
\newblock In  \cite{DBLP:conf/formats/2011}.

\bibitem{DBLP:journals/jql/CortelazzoNT13}
M.~A. Cortelazzo, P.~Nadalutti, and A.~Tuzzi.
\newblock Improving {L}abb{\'e}'s intertextual distance: Testing a revised
  version on a large corpus of italian literature.
\newblock {\em {J. Quant. Linguistics}}, 20(2):125--152, 2013.

\bibitem{DBLP:journals/cacm/Damerau64}
F.~Damerau.
\newblock A technique for computer detection and correction of spelling errors.
\newblock {\em Commun. ACM}, 7(3):171--176, 1964.

\bibitem{DBLP:conf/aplas/FahrenbergL13}
U.~Fahrenberg and A.~Legay.
\newblock Generalized quantitative analysis of metric transition systems.
\newblock In {\em APLAS}, vol. 8301 of {\em {Lect. Notes Comput. Sci.}}
  Springer, 2013.

\bibitem{journal/tcs/FahrenbergL13a}
U.~Fahrenberg and A.~Legay.
\newblock The quantitative linear-time--branching-time spectrum.
\newblock {\em Theor. Comput. Sci.}, 2013.
\newblock Online first. \url{http://dx.doi.org/10.1016/j.tcs.2013.07.030}.

\bibitem{DBLP:conf/fsttcs/FahrenbergLT11}
U.~Fahrenberg, A.~Legay, and C.~R. Thrane.
\newblock The quantitative linear-time--branching-time spectrum.
\newblock In {\em FSTTCS}, vol.~13 of {\em LIPIcs}, 2011.

\bibitem{DBLP:conf/formats/2011}
U.~Fahrenberg and S.~Tripakis, eds.
\newblock {\em Formal Modeling and Analysis of Timed Systems - 9th Int. Conf.},
  vol. 6919 of {\em {Lect. Notes Comput. Sci.}} Springer, 2011.

\bibitem{DBLP:conf/formats/Haverkort11}
B.~R. Haverkort.
\newblock Formal modeling and analysis of timed systems: Technology push or
  market pull?
\newblock In  \cite{DBLP:conf/formats/2011}.

\bibitem{books/KaufmanR90}
L.~Kaufman and P.~J. Rousseeuw.
\newblock {\em Finding Groups in Data: An Introduction to Cluster Analysis}.
\newblock Wiley Interscience. Wiley, 1990.

\bibitem{DBLP:conf/formats/KharmehEM11}
S.~A. Kharmeh, K.~Eder, and D.~May.
\newblock A design-for-verification framework for a configurable
  performance-critical communication interface.
\newblock In  \cite{DBLP:conf/formats/2011}.

\bibitem{journals/nrl/Kuhn55}
H.~W. Kuhn.
\newblock The {H}ungarian method for the assignment problem.
\newblock {\em Naval Research Logistics Quarterly}, 2(1-2):83--97, 1955.

\bibitem{journals/issi/labbe10}
C.~Labb{\'e}.
\newblock {I}ke {A}ntkare, one of the great stars in the scientific firmament.
\newblock {\em ISSI Newsletter}, 6(2):48--52, 2010.
\newblock \url{http://hal.archives-ouvertes.fr/hal-00713564}.

\bibitem{DBLP:journals/jql/LabbeL01}
C.~Labb{\'e} and D.~Labb{\'e}.
\newblock Inter-textual distance and authorship attribution. {C}orneille and
  {M}oli{\`e}re.
\newblock {\em J. Quant. Linguistics}, 8(3):213--231, 2001.

\bibitem{DBLP:journals/lalc/LabbeL06}
C.~Labb{\'e} and D.~Labb{\'e}.
\newblock A tool for literary studies: Intertextual distance and tree
  classification.
\newblock {\em Literary Linguistic Comp.}, 21(3):311--326, 2006.

\bibitem{DBLP:journals/scientometrics/LabbeL13}
C.~Labb{\'e} and D.~Labb{\'e}.
\newblock Duplicate and fake publications in the scientific literature: how
  many {SCIgen} papers in computer science?
\newblock {\em Scientometrics}, 94(1):379--396, 2013.

\bibitem{DBLP:journals/jql/Labbe07}
D.~Labb{\'e}.
\newblock Experiments on authorship attribution by intertextual distance in
  {E}nglish.
\newblock {\em {J. Quant. Linguistics}}, 14(1):33--80, 2007.

\bibitem{DBLP:conf/naacl/LinH03}
C.-Y. Lin and E.~H. Hovy.
\newblock Automatic evaluation of summaries using n-gram co-occurrence
  statistics.
\newblock In {\em HLT-NAACL}, 2003.

\bibitem{DBLP:conf/acl/LinO04}
C.-Y. Lin and F.~J. Och.
\newblock Automatic evaluation of machine translation quality using longest
  common subsequence and skip-bigram statistics.
\newblock In {\em ACL}. ACL, 2004.

\bibitem{journals/jmb/NeedlemanW70}
S.~B. Needleman and C.~D. Wunsch.
\newblock A general method applicable to the search for similarities in the
  amino acid sequence of two proteins.
\newblock {\em J. Molecular Bio.}, 48(3):443--453, 1970.

\bibitem{news/nature/Gibberish14}
R.~V. Noorden.
\newblock Publishers withdraw more than 120 gibberish papers.
\newblock Nature News \& Comment, Feb. 2014.
\newblock \url{http://dx.doi.org/10.1038/nature.2014.14763}.

\bibitem{DBLP:conf/acl/PapineniRWZ02}
K.~Papineni, S.~Roukos, T.~Ward, and W.-J. Zhu.
\newblock {BLEU}: a method for automatic evaluation of machine translation.
\newblock In {\em ACL}. ACL, 2002.

\bibitem{DBLP:conf/formats/SankaranarayananHL11}
S.~Sankaranarayanan, H.~Homaei, and C.~Lewis.
\newblock Model-based dependability analysis of programmable drug infusion
  pumps.
\newblock In  \cite{DBLP:conf/formats/2011}.

\bibitem{DBLP:journals/jql/Savoy12}
J.~Savoy.
\newblock Authorship attribution: A comparative study of three text corpora and
  three languages.
\newblock {\em {J. Quant. Linguistics}}, 19(2):132--161, 2012.

\bibitem{DBLP:journals/tois/Savoy12}
J.~Savoy.
\newblock Authorship attribution based on specific vocabulary.
\newblock {\em ACM Trans. Inf. Syst.}, 30(2):12, 2012.

\bibitem{DBLP:journals/corr/SmithKSBP13v1}
S.~T. Smith, E.~K. Kao, K.~D. Senne, G.~Bernstein, and S.~Philips.
\newblock Bayesian discovery of threat networks.
\newblock {\em CoRR}, abs/1311.5552v1, 2013.

\bibitem{DBLP:journals/corr/abs-1303-5613v1}
S.~T. Smith, K.~D. Senne, S.~Philips, E.~K. Kao, and G.~Bernstein.
\newblock Network detection theory and performance.
\newblock {\em CoRR}, abs/1303.5613v1, 2013.

\bibitem{journals/jmb/SmithW81}
T.~Smith and M.~Waterman.
\newblock Identification of common molecular subsequences.
\newblock {\em J. Molecular Bio.}, 147(1):195--197, 1981.

\bibitem{news/springer/SCIgen2update}
Springer second update on {SCIgen}-generated papers in conference proceedings.
\newblock Springer Statement, Apr. 2014.
\newblock
  \url{http://www.springer.com/about+springer/media/statements?SGWID=0-1760813%
-6-1460747-0}.

\bibitem{conf/dhcase/TomasiBCDGV13}
F.~Tomasi, I.~Bartolini, F.~Condello, M.~Degli~Esposti, V.~Garulli, and
  M.~Viale.
\newblock Towards a taxonomy of suspected forgery in authorship attribution
  field. {A} case: {M}ontale's {D}iario {P}ostumo.
\newblock In {\em DH-CASE}. ACM, 2013.

\bibitem{DBLP:journals/corr/abs-1202-1307v2}
A.~Ulusoy, S.~L. Smith, X.~C. Ding, and C.~Belta.
\newblock Robust multi-robot optimal path planning with temporal logic
  constraints.
\newblock {\em CoRR}, abs/1202.1307v2, 2012.

\bibitem{DBLP:journals/corr/abs-1107-0062v1}
A.~Ulusoy, S.~L. Smith, X.~C. Ding, C.~Belta, and D.~Rus.
\newblock Optimal multi-robot path planning with temporal logic constraints.
\newblock {\em CoRR}, abs/1107.0062v1, 2011.

\end{thebibliography}

\newpage
\appendix

\section*{Appendix}

\subsection*{Proof of Theorem~\ref{th:labbe}}

Let $\phi$ be an optimal matching in $\Dphrase$ and let $\dphrase( i,
\phi( i), 0)= 1$ (if such $i$ does not exist, then $d_2( A, B, 0)= 0$
and we are done).  Let $w= a_i$.  Assume that there is $b_j= w$ for
which $\dphrase( \phi^{ -1}( j), j)= 1$, then we can define a new
permutation $\phi'$ by $\phi'( i)= j$ and $\phi'( \phi^{ -1}( j))= \phi(
i)$ (and otherwise, values like $\phi$), and $\phi'$ is a better
matching than $\phi$, a contradiction.

Hence $\dphrase( \phi^{ -1}( j), j)= 0$ for all $j$ such that $b_j= w$.
In other words, $\dphrase( i, \phi( i), 0)= 1$ marks the fact that the
word $w= a_i$ occurs one time more in $A$ than in $B$.  The same holds
for all other indices $i$ for which $w= a_i$ and $\dphrase( i, \phi( i),
0)= 1$, so that
\begin{equation*}
  | F_A( w)- F_B( w)|=|\{ i\mid a_i= w, \dphrase( i, \phi( i), 0)=
  1\}|
\end{equation*}
in this case.

Similarly, if we let $v= b_{ \phi( i)}$, then $\dphrase( i, \phi( i),
0)= 1$ marks the fact that the word $v$ occurs one time more in $B$ than
in $A$.  Collecting these two, we see that
\begin{multline*}
  | F_A( w)- F_B( w)|=|\{ i\mid a_i= w, \dphrase( i, \phi( i), 0)= 1\}|
  \\ +|\{ j\mid b_j= w, \dphrase( \phi^{ -1}( j), j, 0)= 1\}|
\end{multline*}
for all words $w\in A\circ B$.  Thus
\begin{align*}
  \sum_{ w\in A\circ B}| F_A( w)- F_B( w)| &= \sum_{ w\in A\circ B}|\{
  i\mid a_i= w, \dphrase( i, \phi( i), 0)= 1\}| \\
  &\qquad+ \sum_{ w\in A\circ B}|\{
  j\mid b_j= w, \dphrase( \phi^{ -1}( j), j, 0)= 1\}| \\
  &= |\{ i\mid \dphrase( i, \phi( i), 0)= 1\}| \\
  &\qquad+ |\{ j\mid \dphrase( \phi^{ -1}( j), j, 0)= 1\}| \\
  &= \sum_i \dphrase( i, \phi( i), 0)+ \sum_j \dphrase( \phi^{ -1}(
  j), j, 0) \\
  &= 2 \sum_i \dphrase( i, \phi( i), 0)\,,
\end{align*}
so that
\begin{align*}
  \dlabbe( A, B) &= \frac{ \sum_{ w\in A\circ B}| F_A( w)- F_B( w)|}{
    2 N_A} \\
  &= \frac{ \sum_i \dphrase( i, \phi( i), 0)}{ N_A}= d_2( A, B, 0)\,,
\end{align*}
because $\phi$ was assumed optimal.

\begin{figure}
  \hspace*{-2em}
  \includegraphics[width=.6\linewidth, trim=30 70 0 70, clip, angle=-90]{experiment_1/plots/dendrograms/average/dendrogramm_0}
  \hfill
  \includegraphics[width=.6\linewidth, trim=30 70 0 70, clip, angle=-90]{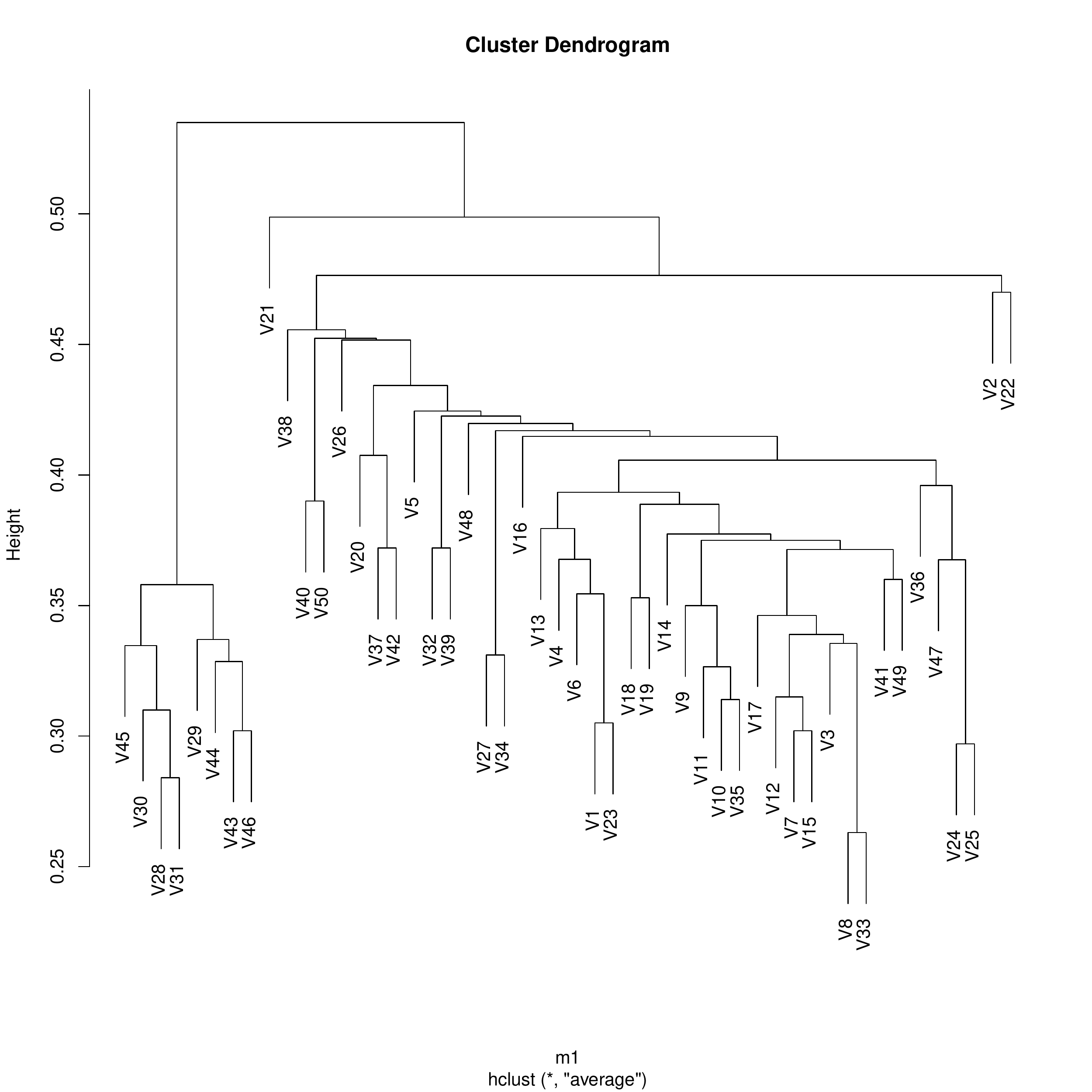}
  \hspace*{-2em}
  
  \bigskip%
  \hspace*{-2em}
  \includegraphics[width=.6\linewidth, trim=30 70 0 70, clip, angle=-90]{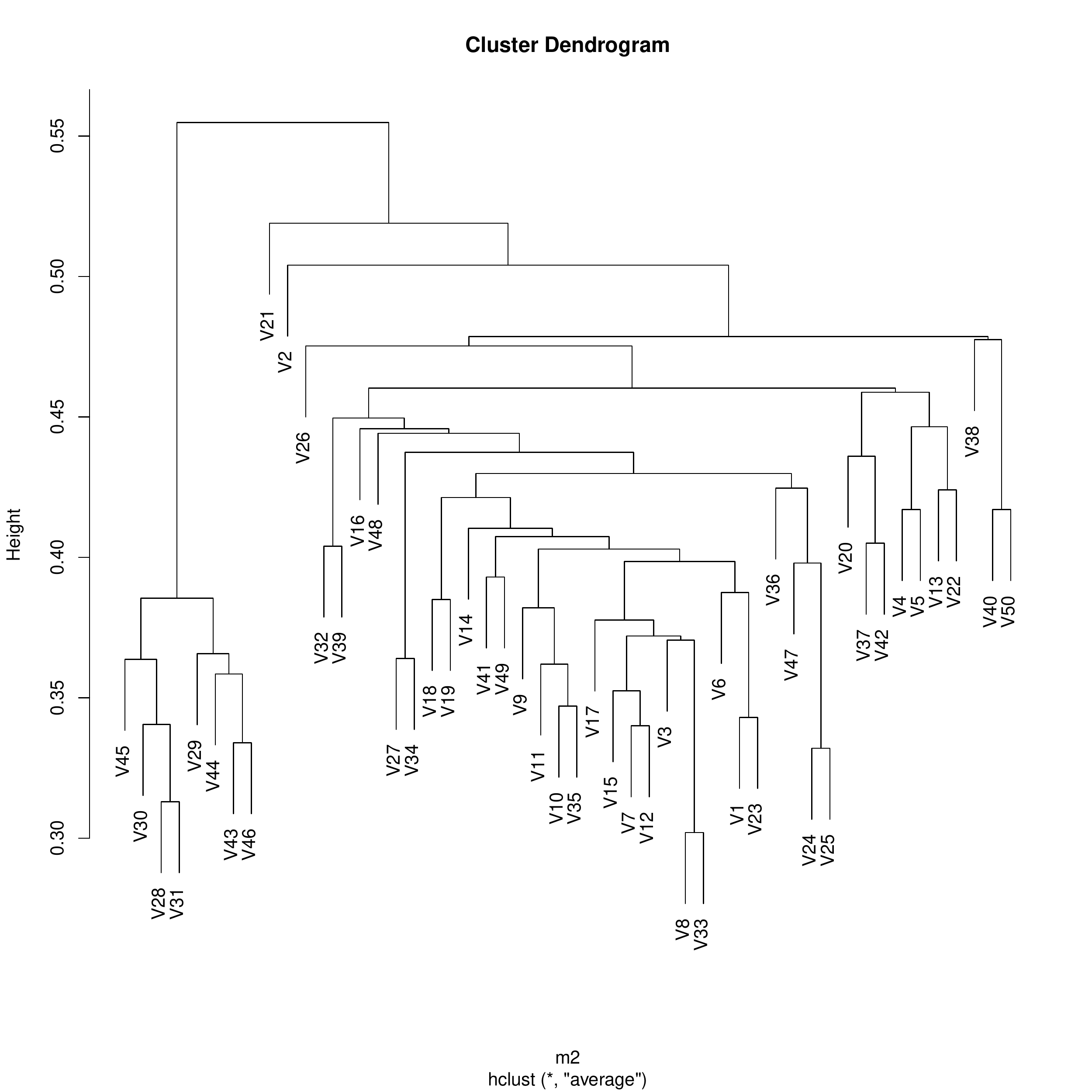}
  \hfill
  \includegraphics[width=.6\linewidth, trim=30 65 0 75, clip, angle=-90]{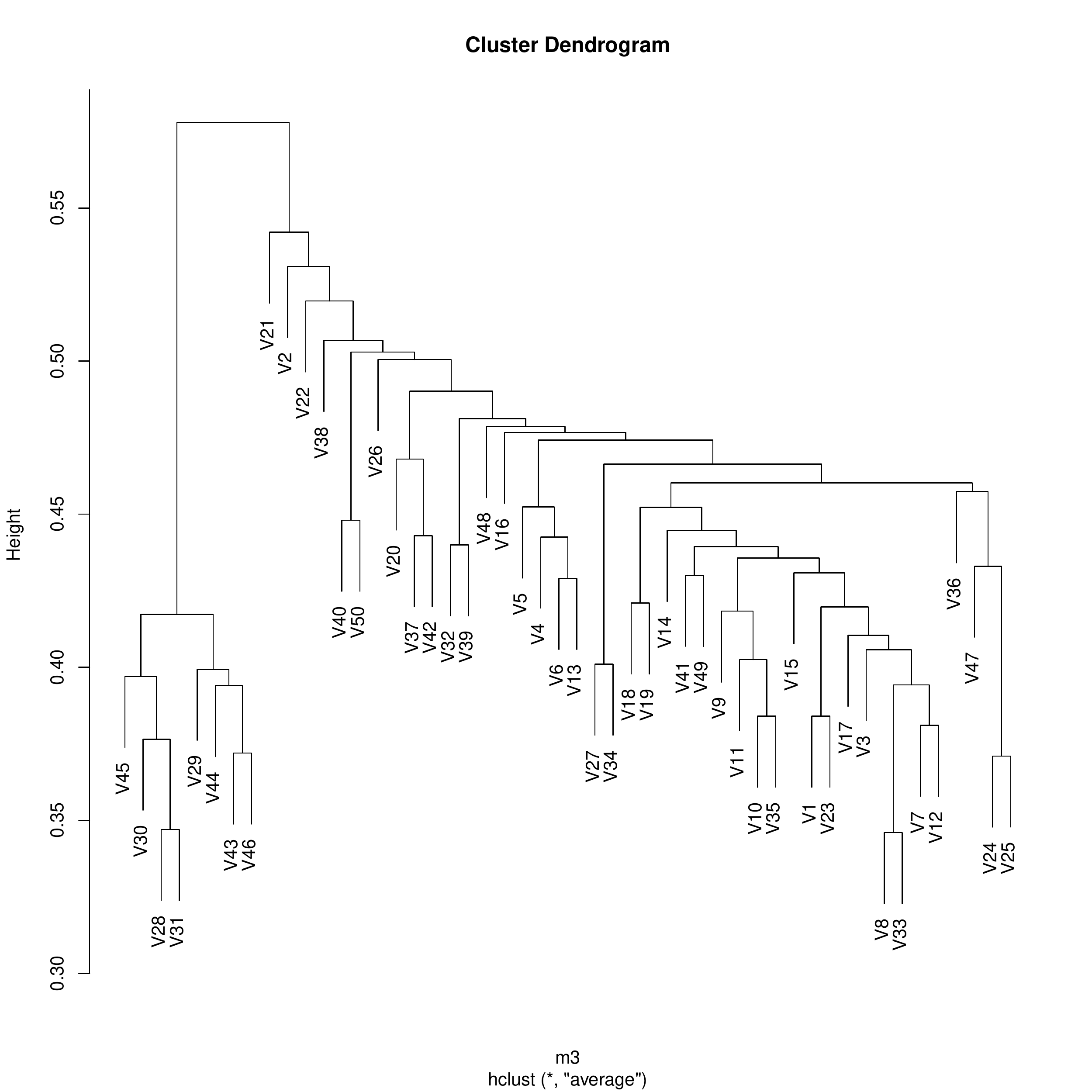}
  \hspace*{-2em}
  \caption{%
    \label{fi:dendro-1-avg.1}
    Dendrograms for Experiment 1, using average clustering, for
    discounting factors $0$ (top left), $.1$ (top right), $.2$ (bottom
    left) and $.3$ (bottom right), respectively.  Fake papers are
    numbered 28-31 (Antkare) and 43-46 (SCIgen), the others are
    genuine.}
\end{figure}

\begin{figure}
  \hspace*{-2em}
  \includegraphics[width=.6\linewidth, trim=30 70 0 70, clip, angle=-90]{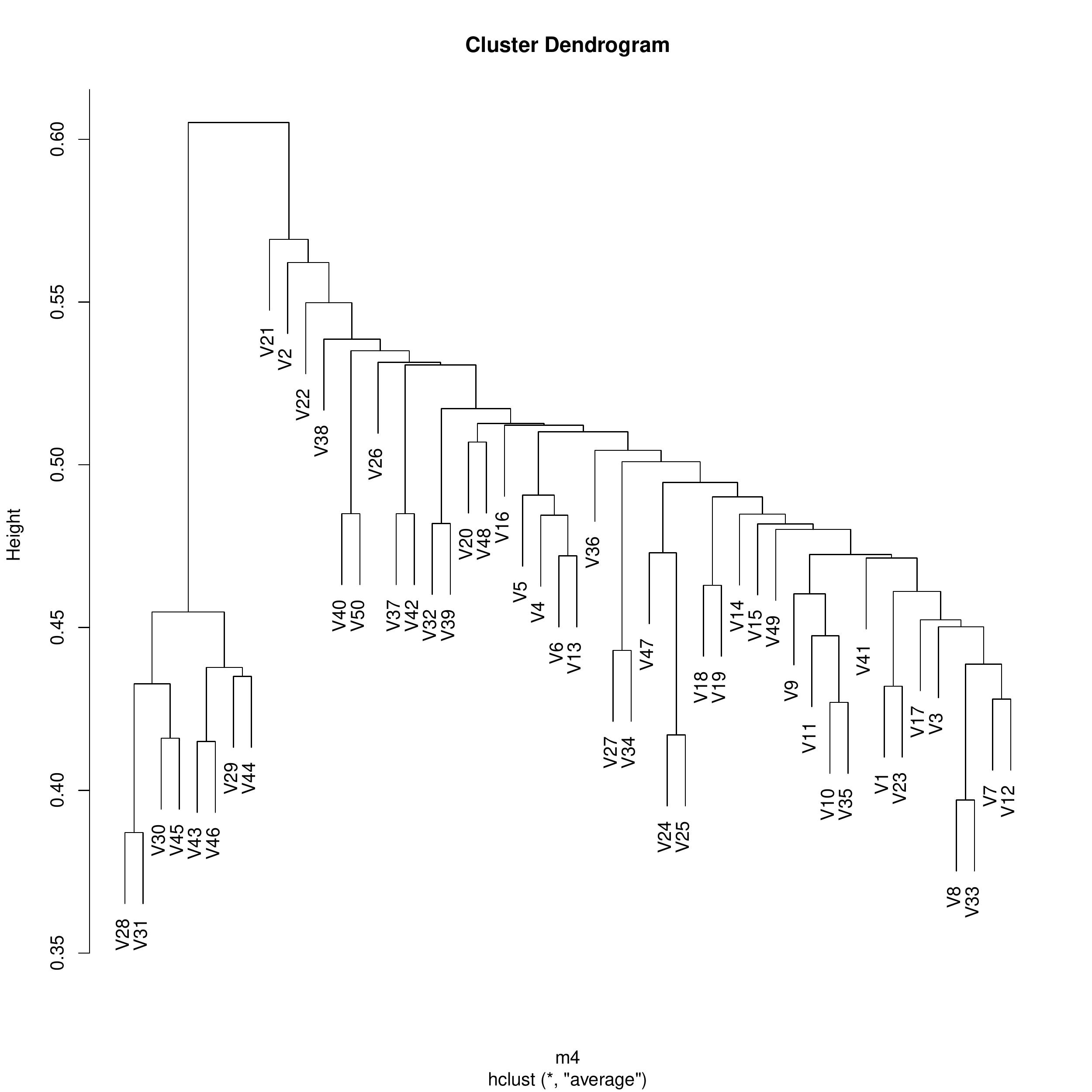}
  \hfill
  \includegraphics[width=.6\linewidth, trim=30 70 0 70, clip, angle=-90]{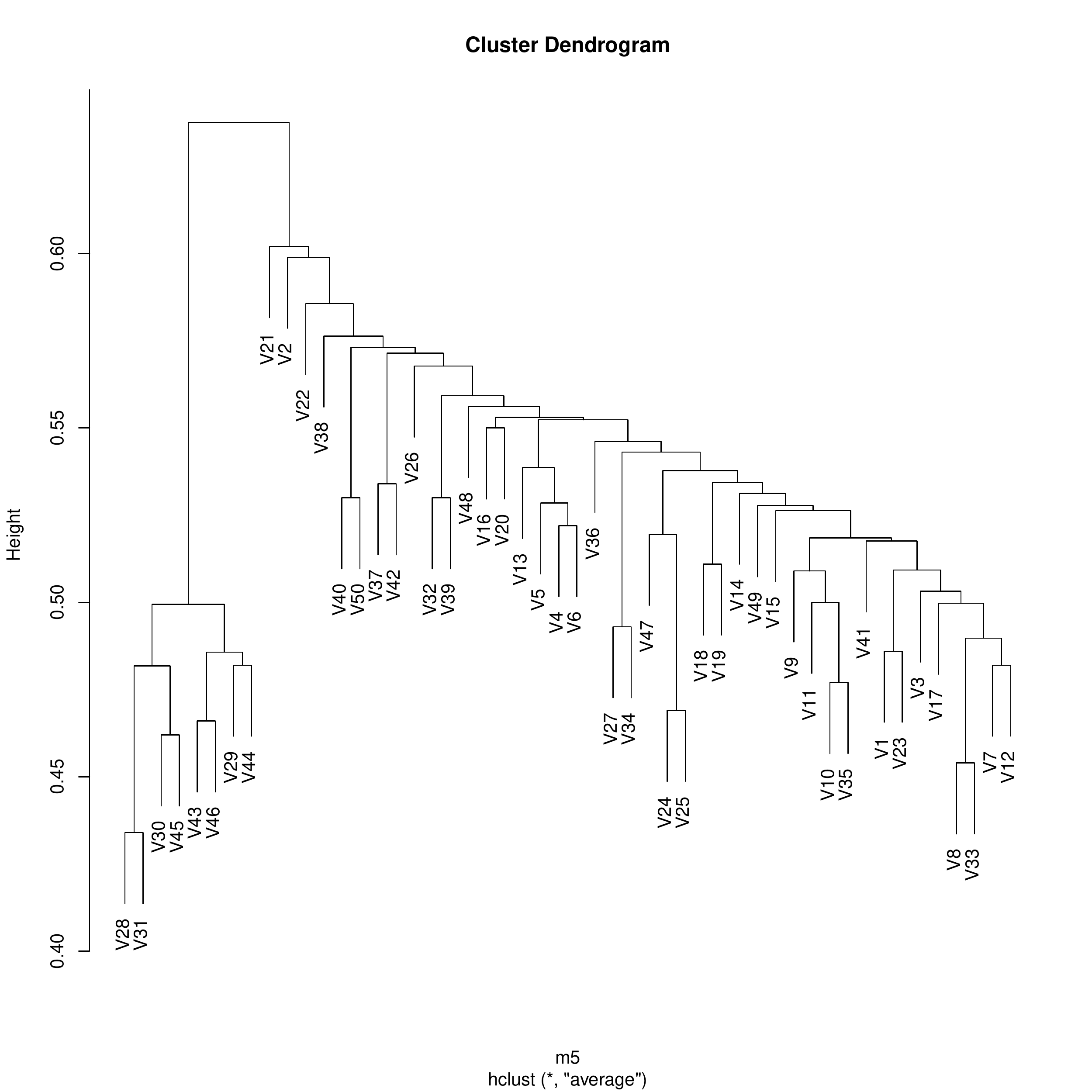}
  \hspace*{-2em}
  
  \bigskip%
  \hspace*{-2em}
  \includegraphics[width=.6\linewidth, trim=30 65 0 75, clip, angle=-90]{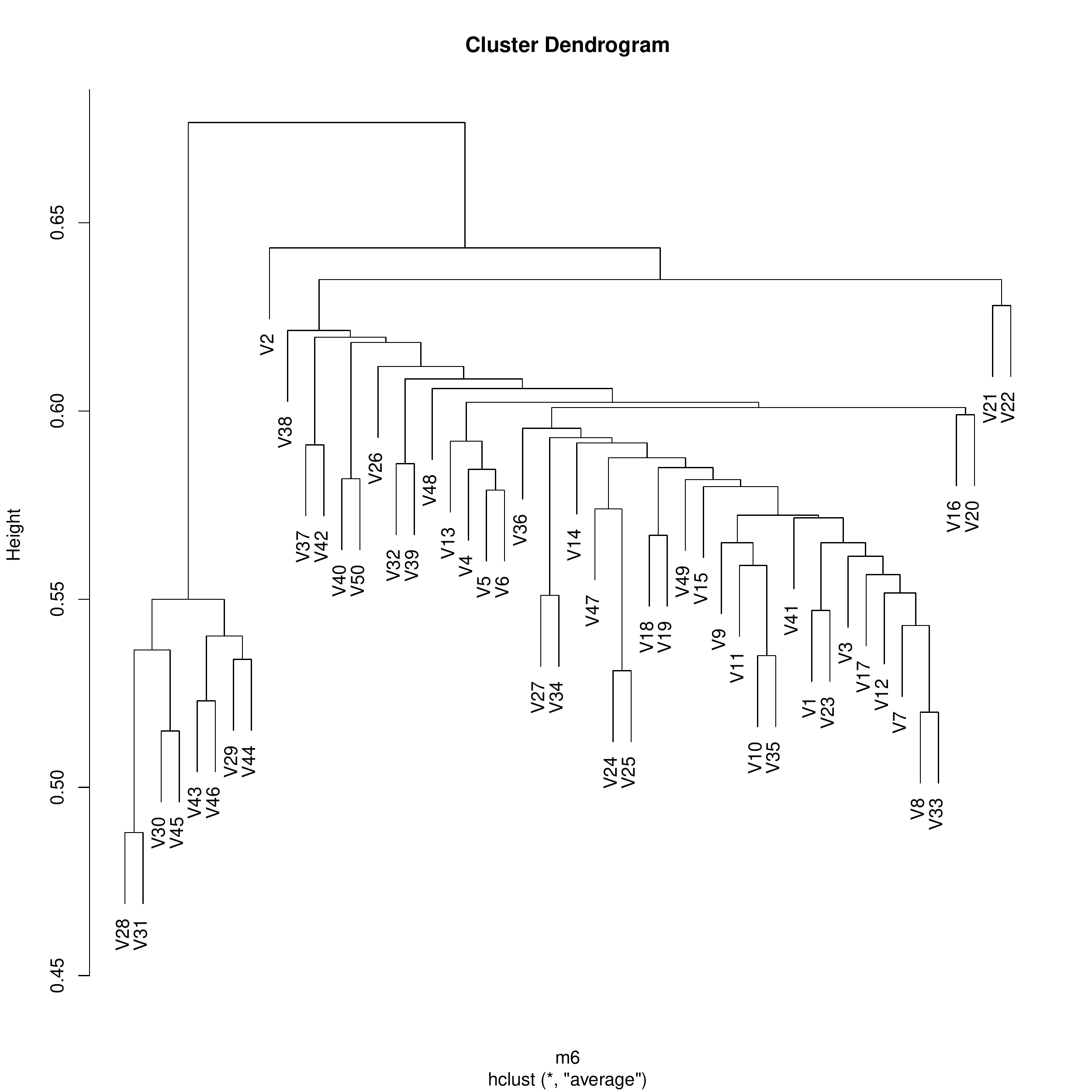}
  \hfill
  \includegraphics[width=.6\linewidth, trim=30 65 0 75, clip, angle=-90]{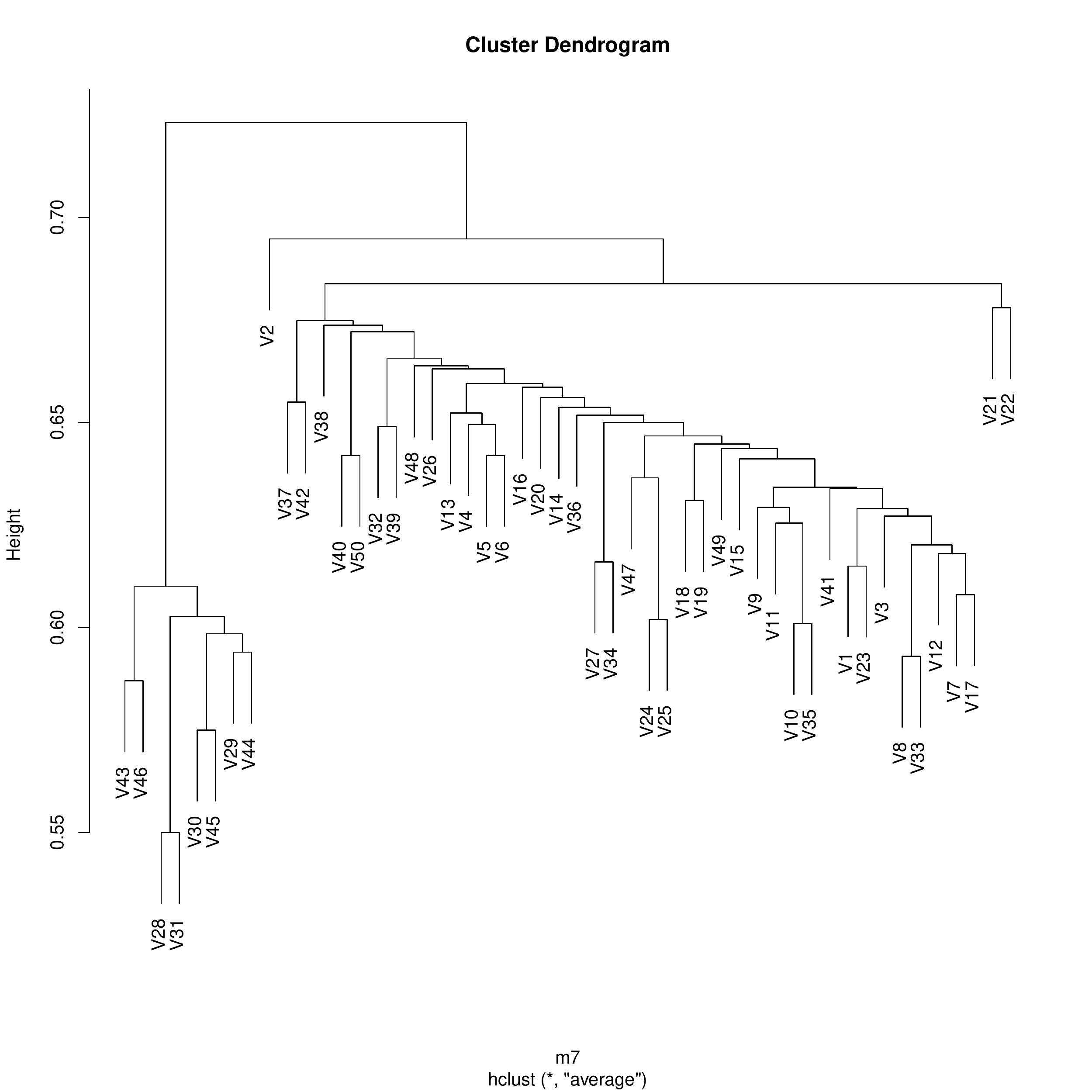}
  \hspace*{-2em}
  \caption{%
    \label{fi:dendro-1-avg.2}
    Dendrograms for Experiment 1, using average clustering, for
    discounting factors $.4$ (top left), $.5$ (top right), $.6$ (bottom
    left) and $.7$ (bottom right), respectively.  Fake papers are
    numbered 28-31 (Antkare) and 43-46 (SCIgen), the others are
    genuine.}
\end{figure}

\begin{figure}
  \hspace*{-2em}
  \includegraphics[width=.6\linewidth, trim=30 70 0 70, clip, angle=-90]{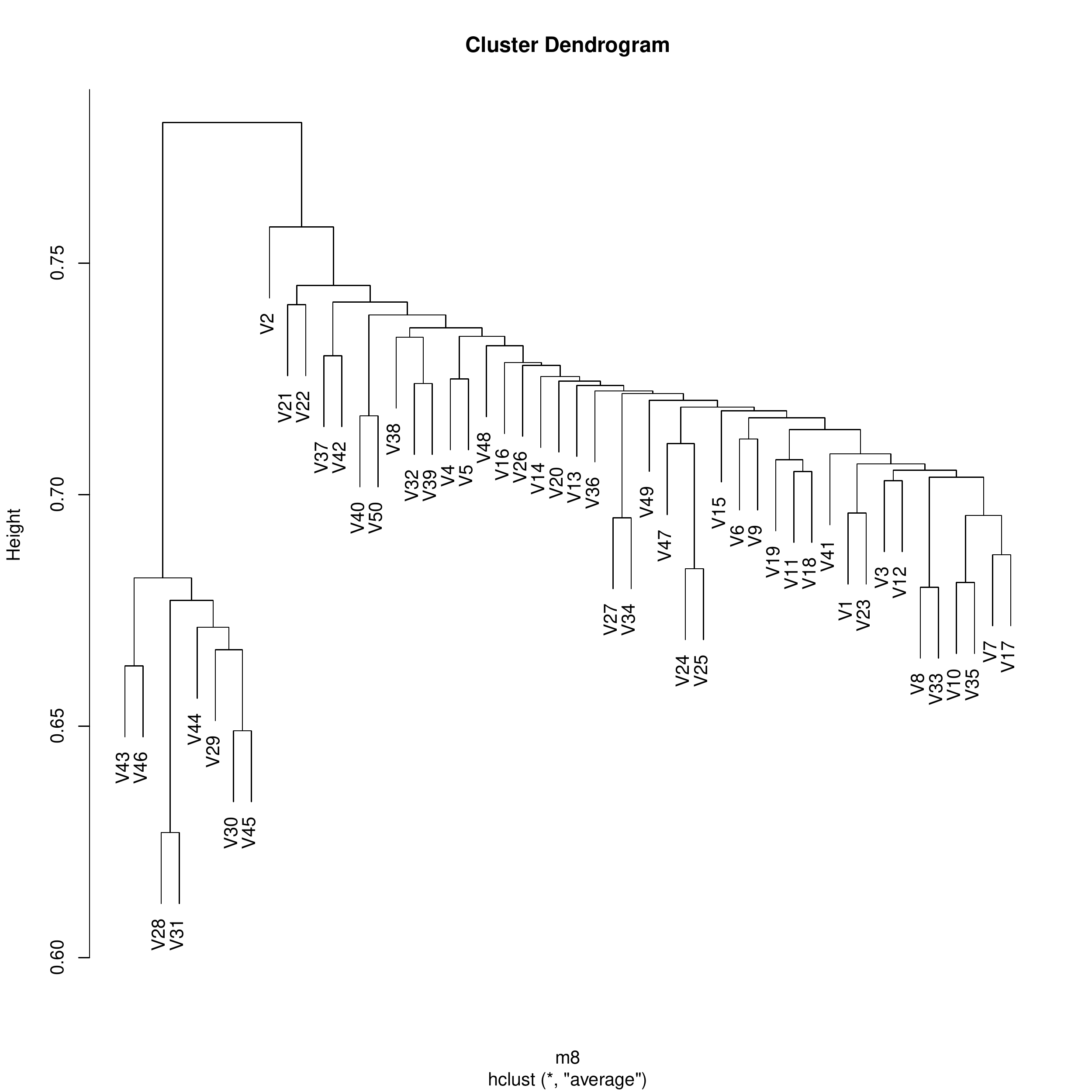}
  \hfill
  \includegraphics[width=.6\linewidth, trim=30 60 0 80, clip, angle=-90]{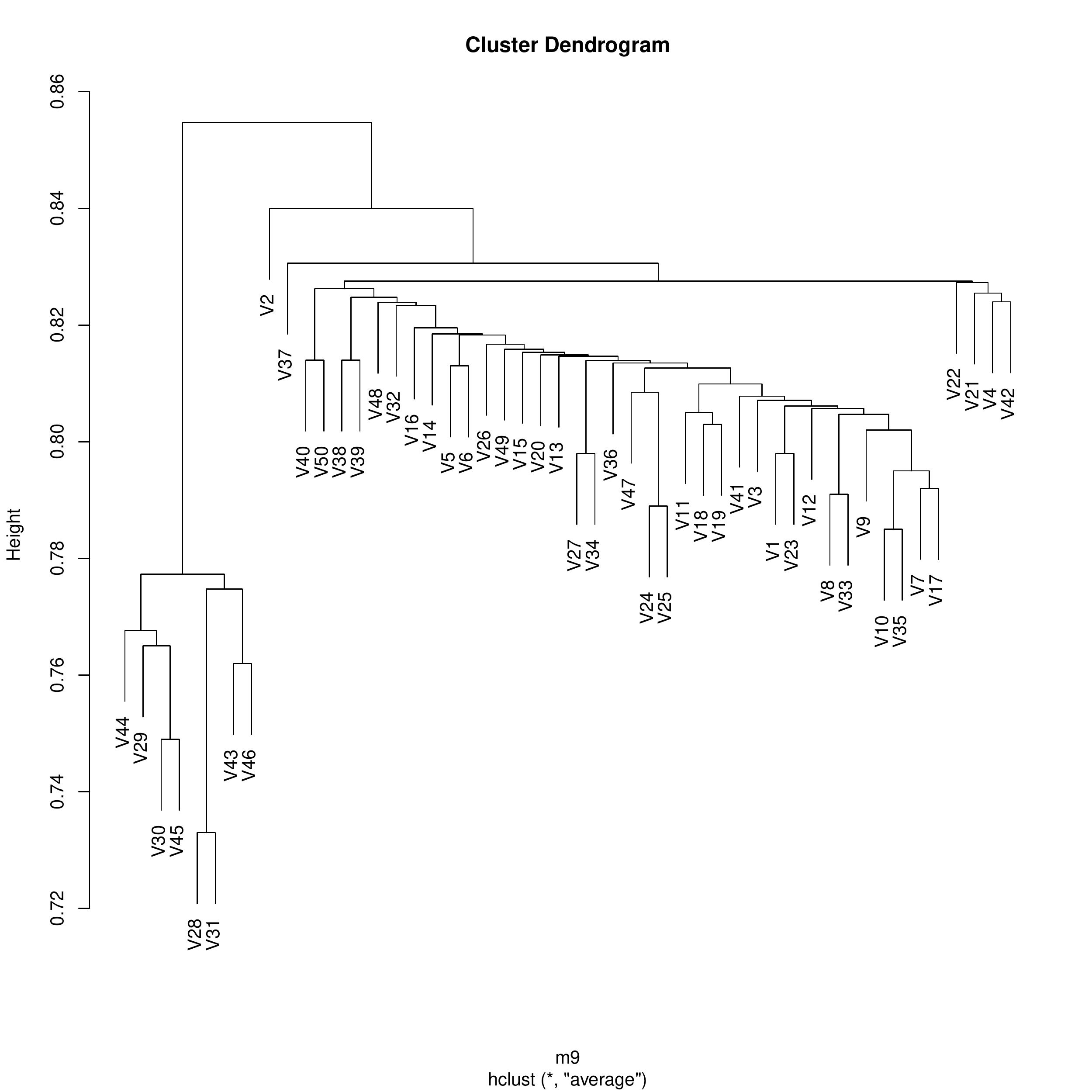}
  \hspace*{-2em}
  
  \bigskip%
  \hspace*{-2em}
  \includegraphics[width=.6\linewidth, trim=30 70 0 70, clip, angle=-90]{experiment_1/plots/dendrograms/average/dendrogramm_95}
  \caption{%
    \label{fi:dendro-1-avg.3}
    Dendrograms for Experiment 1, using average clustering, for
    discounting factors $.8$ (top left), $.9$ (top right) and $.95$
    (bottom), respectively.  Fake papers are numbered 28-31 (Antkare)
    and 43-46 (SCIgen), the others are genuine.}
\end{figure}

\begin{figure}
  \hspace*{-2em}
  \includegraphics[width=.6\linewidth, trim=30 70 0 50, clip, angle=-90]{experiment_1/plots/dendrograms/ward/ward_0}
  \hfill
  \includegraphics[width=.6\linewidth, trim=30 70 0 70, clip, angle=-90]{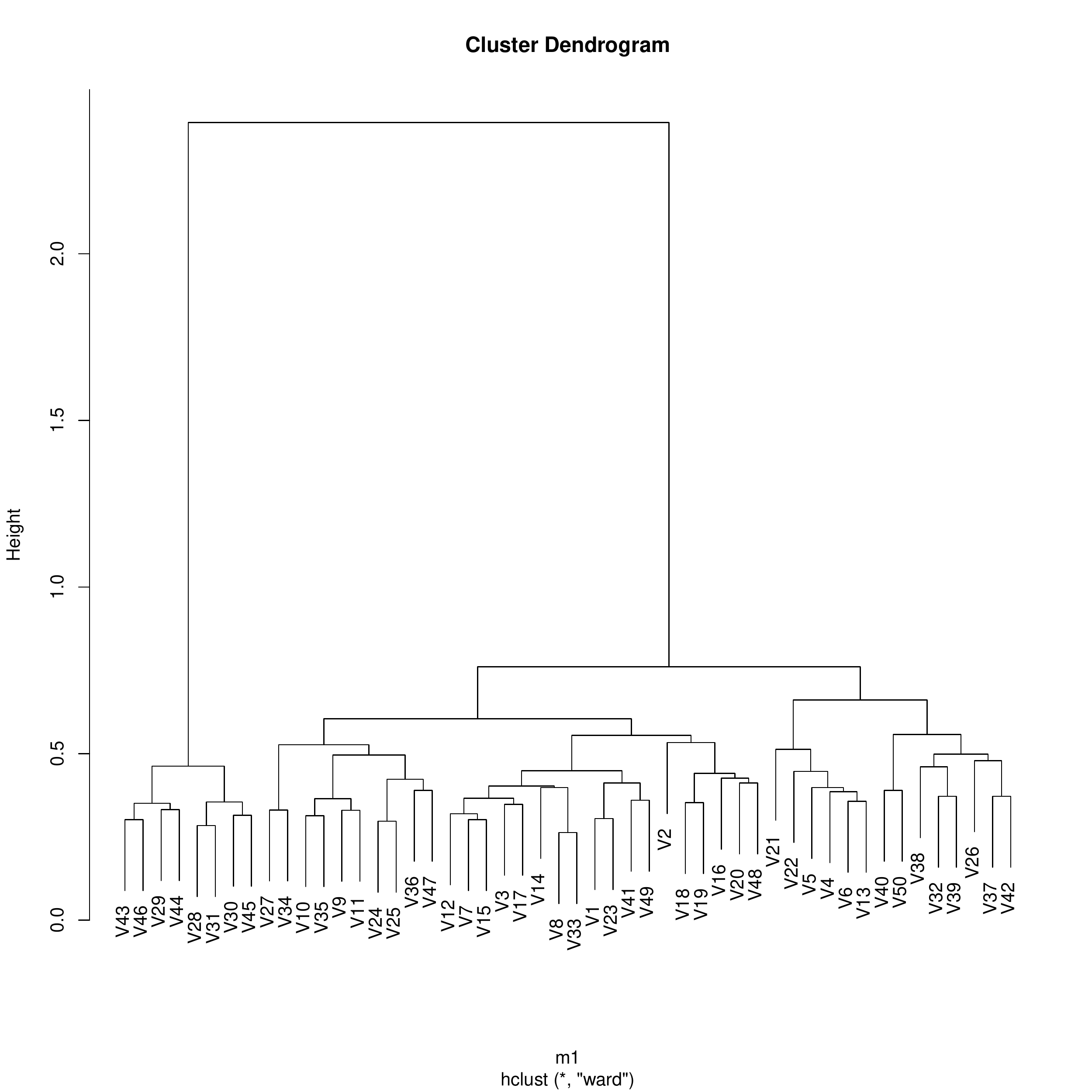}
  \hspace*{-2em}

  \medskip%
  \hspace*{-2em}
  \includegraphics[width=.6\linewidth, trim=30 70 0 50, clip, angle=-90]{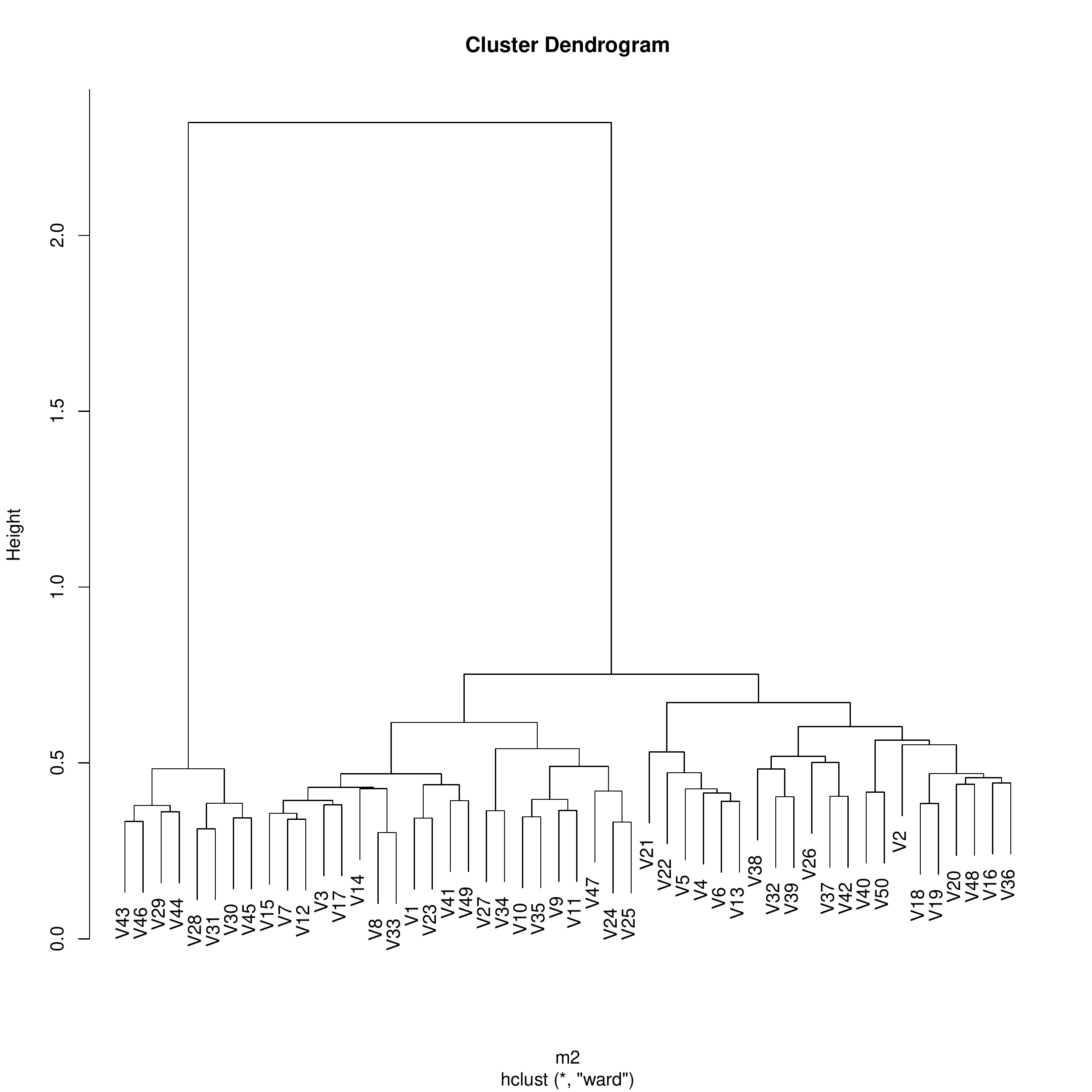}
  \hfill
  \includegraphics[width=.6\linewidth, trim=30 70 0 70, clip, angle=-90]{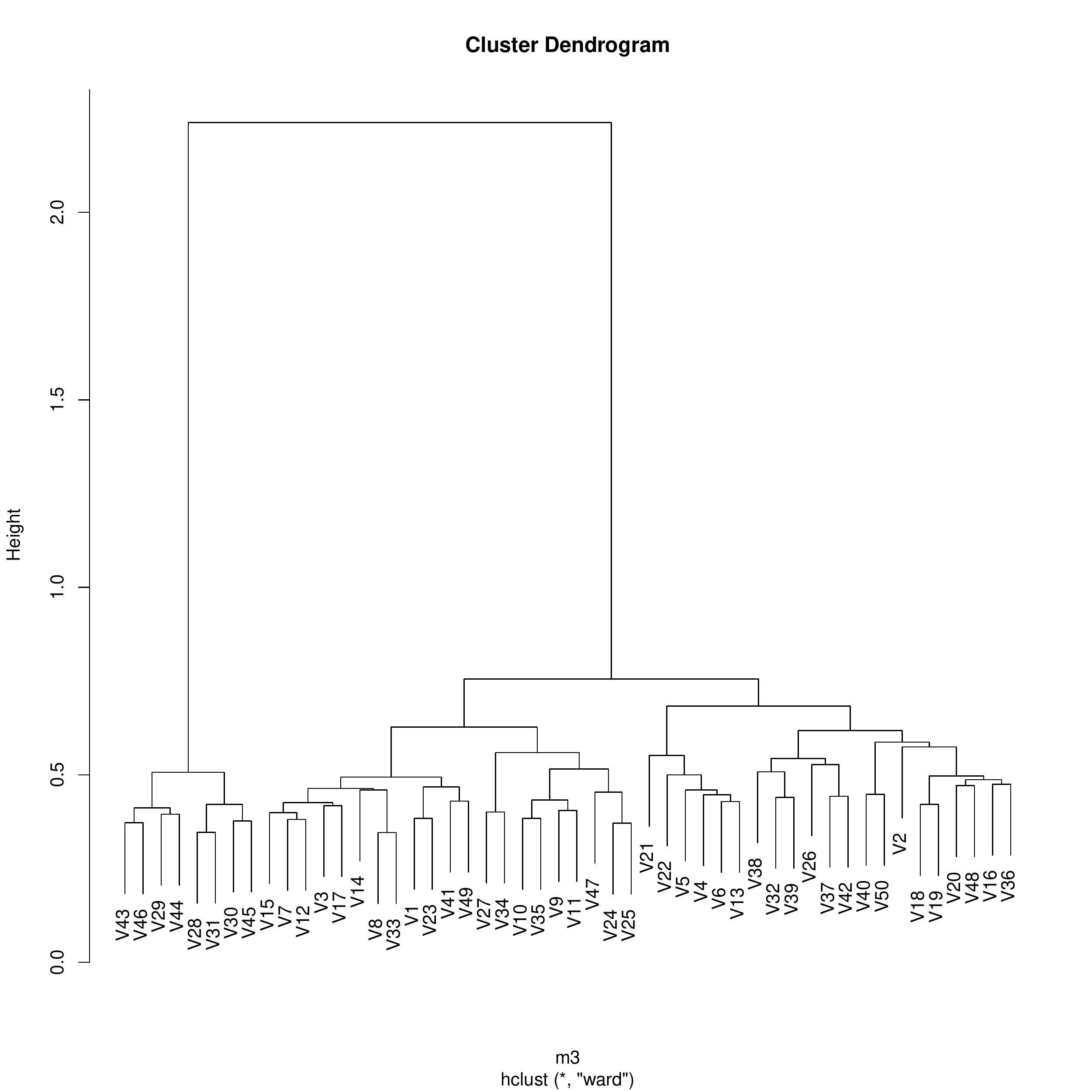}
  \hspace*{-2em}
  \caption{%
    \label{fi:dendro-1-ward.1}
    Dendrograms for Experiment 1, using Ward clustering, for discounting
    factors $0$ (top left), $.1$ (top right), $.2$ (bottom left) and
    $.3$ (bottom right), respectively.  Fake papers are numbered 28-31
    (Antkare) and 43-46 (SCIgen), the others are genuine.}
\end{figure}

\begin{figure}
  \hspace*{-2em}
  \includegraphics[width=.6\linewidth, trim=30 70 0 50, clip, angle=-90]{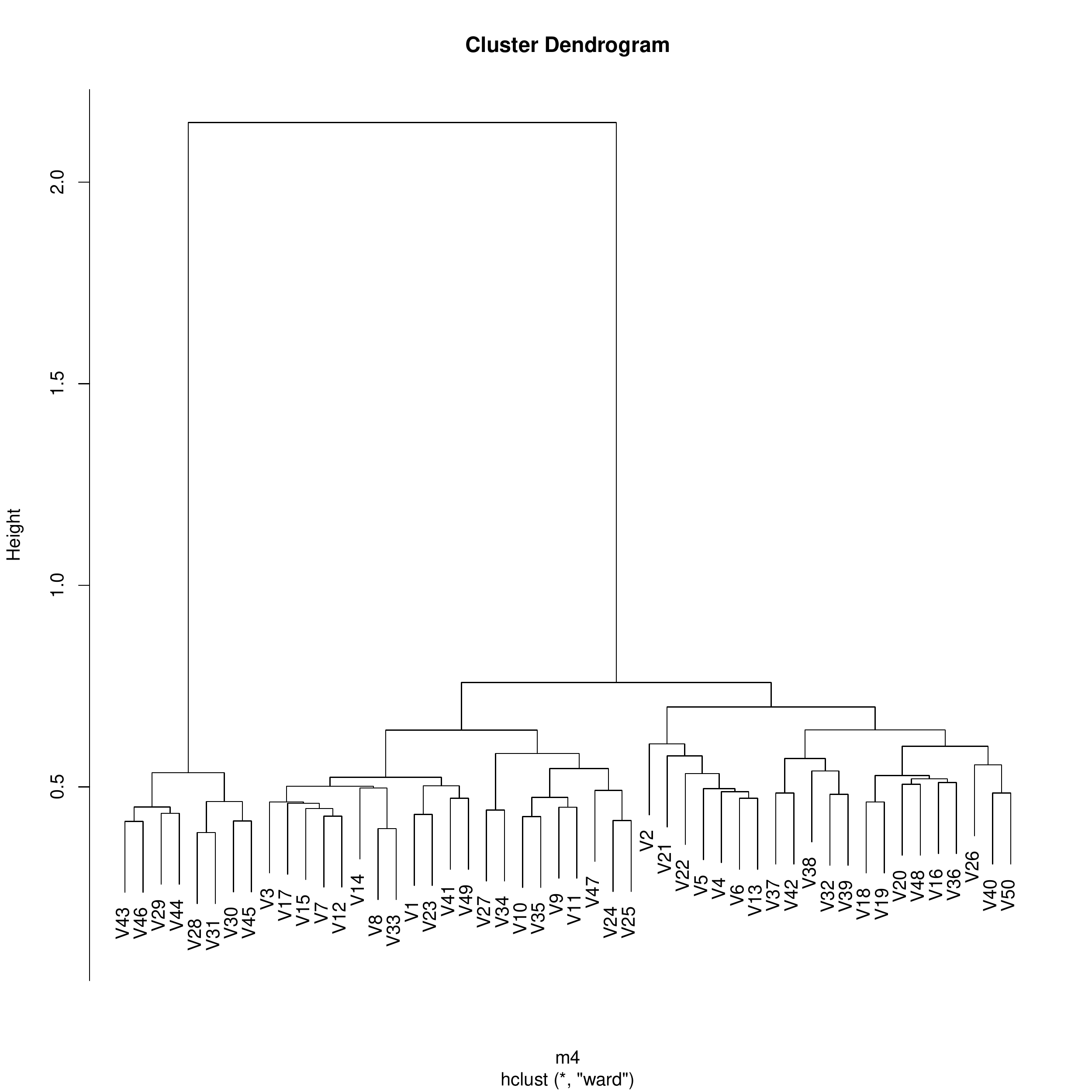}
  \hfill
  \includegraphics[width=.6\linewidth, trim=30 70 0 70, clip, angle=-90]{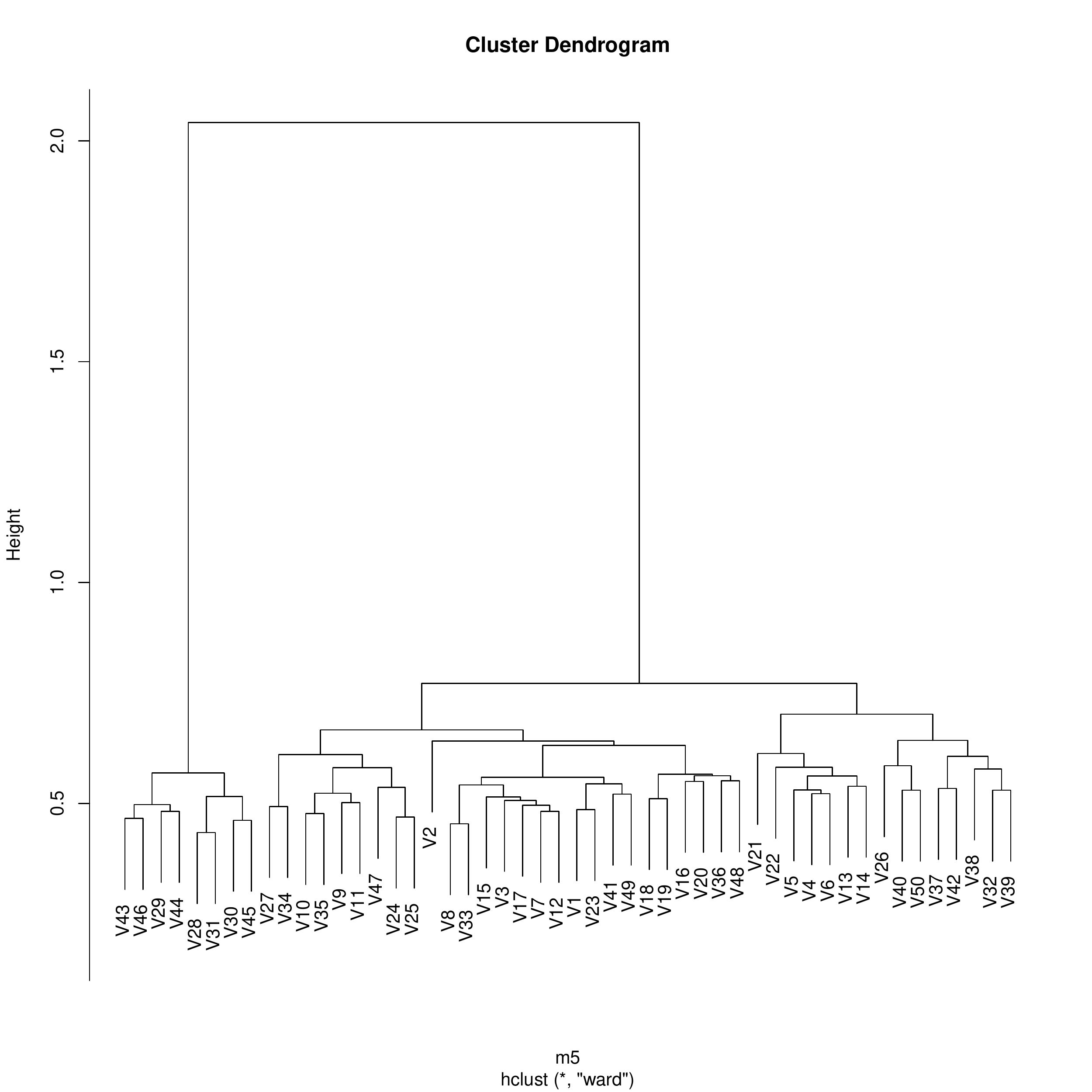}
  \hspace*{-2em}

  \medskip%
  \hspace*{-2em}
  \includegraphics[width=.6\linewidth, trim=30 70 0 50, clip, angle=-90]{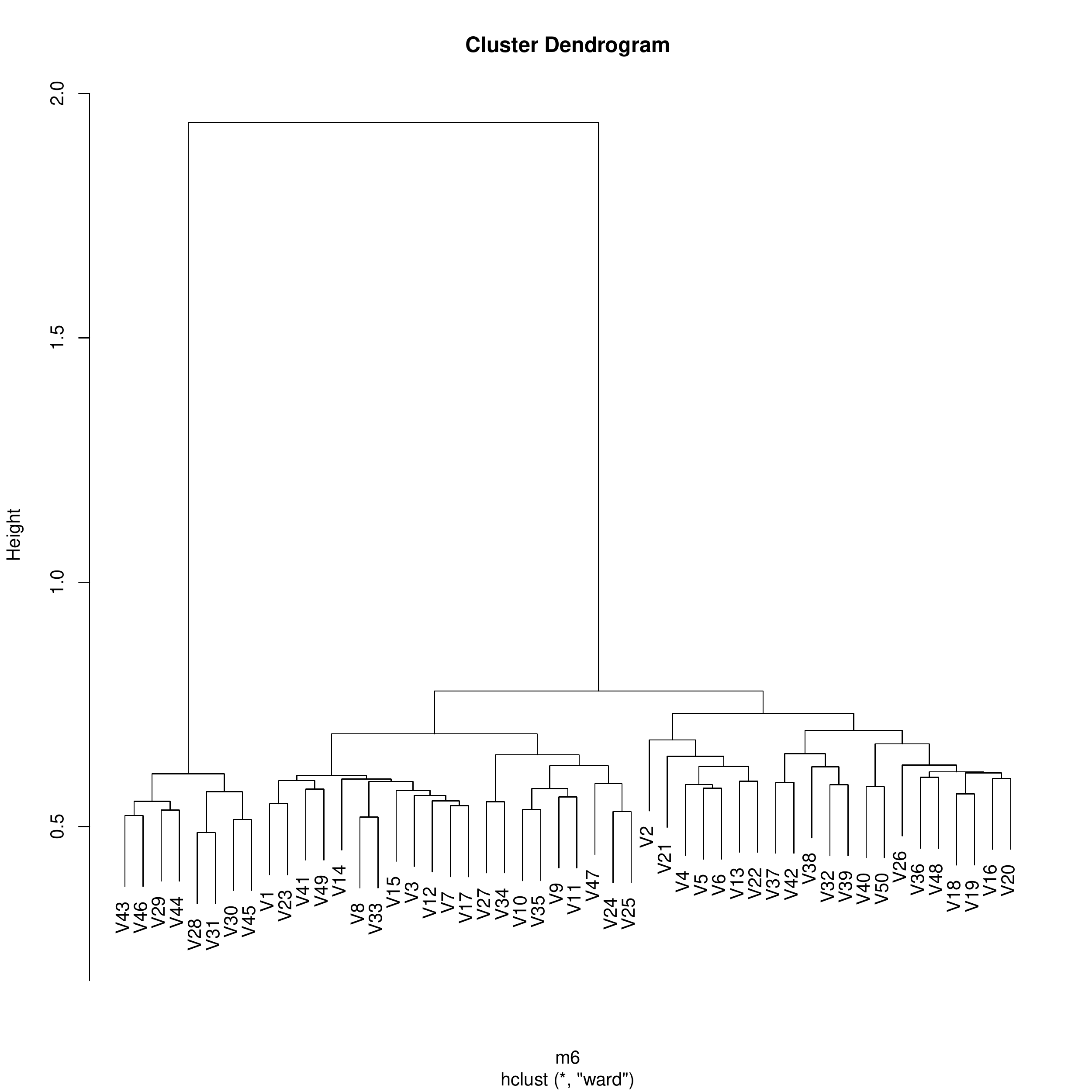}
  \hfill
  \includegraphics[width=.6\linewidth, trim=30 70 0 70, clip, angle=-90]{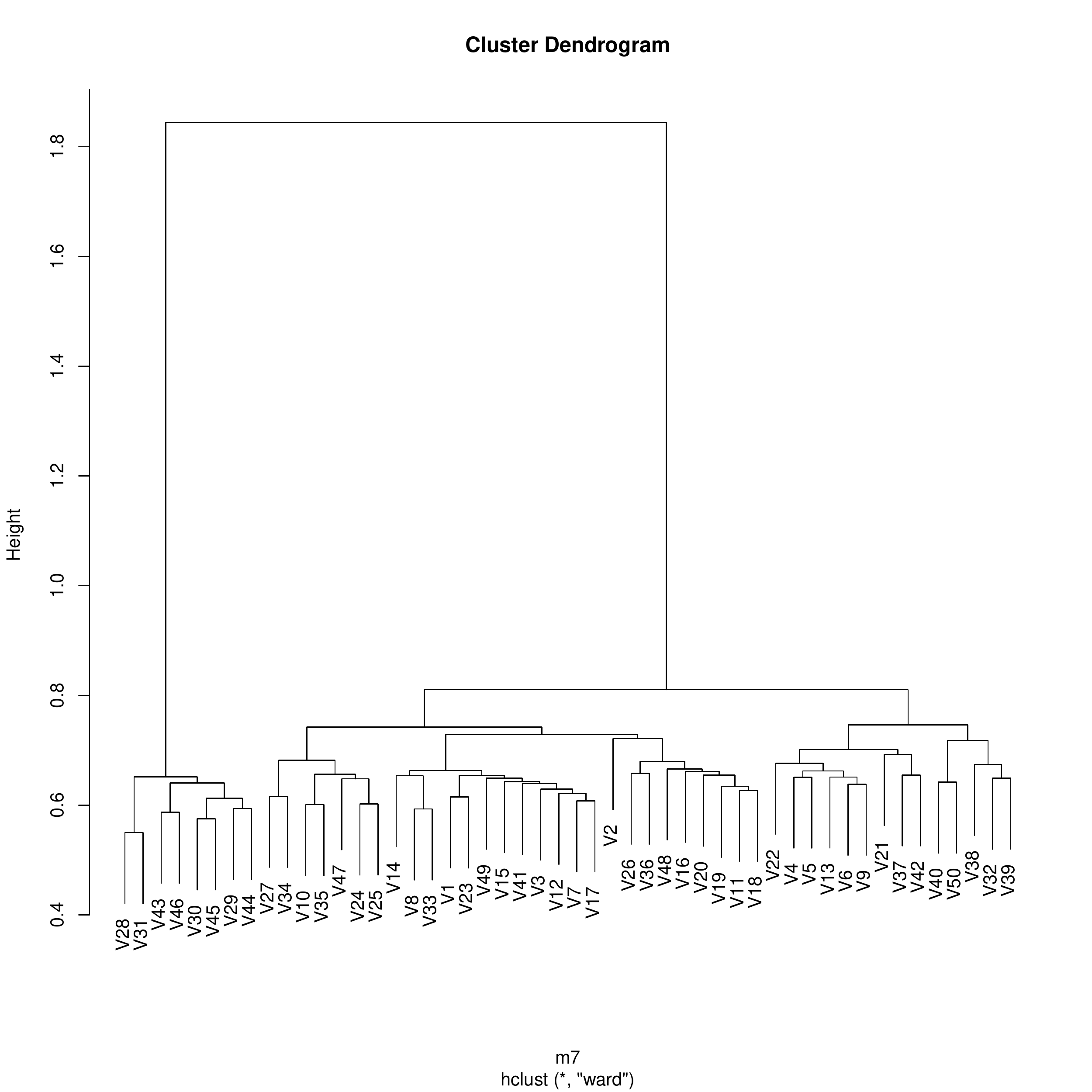}
  \hspace*{-2em}
  \caption{%
    \label{fi:dendro-1-ward.2}
    Dendrograms for Experiment 1, using Ward clustering, for discounting
    factors $.4$ (top left), $.5$ (top right), $.6$ (bottom left) and
    $.7$ (bottom right), respectively.  Fake papers are numbered 28-31
    (Antkare) and 43-46 (SCIgen), the others are genuine.}
\end{figure}

\begin{figure}
  \hspace*{-2em}
  \includegraphics[width=.6\linewidth, trim=30 70 0 50, clip, angle=-90]{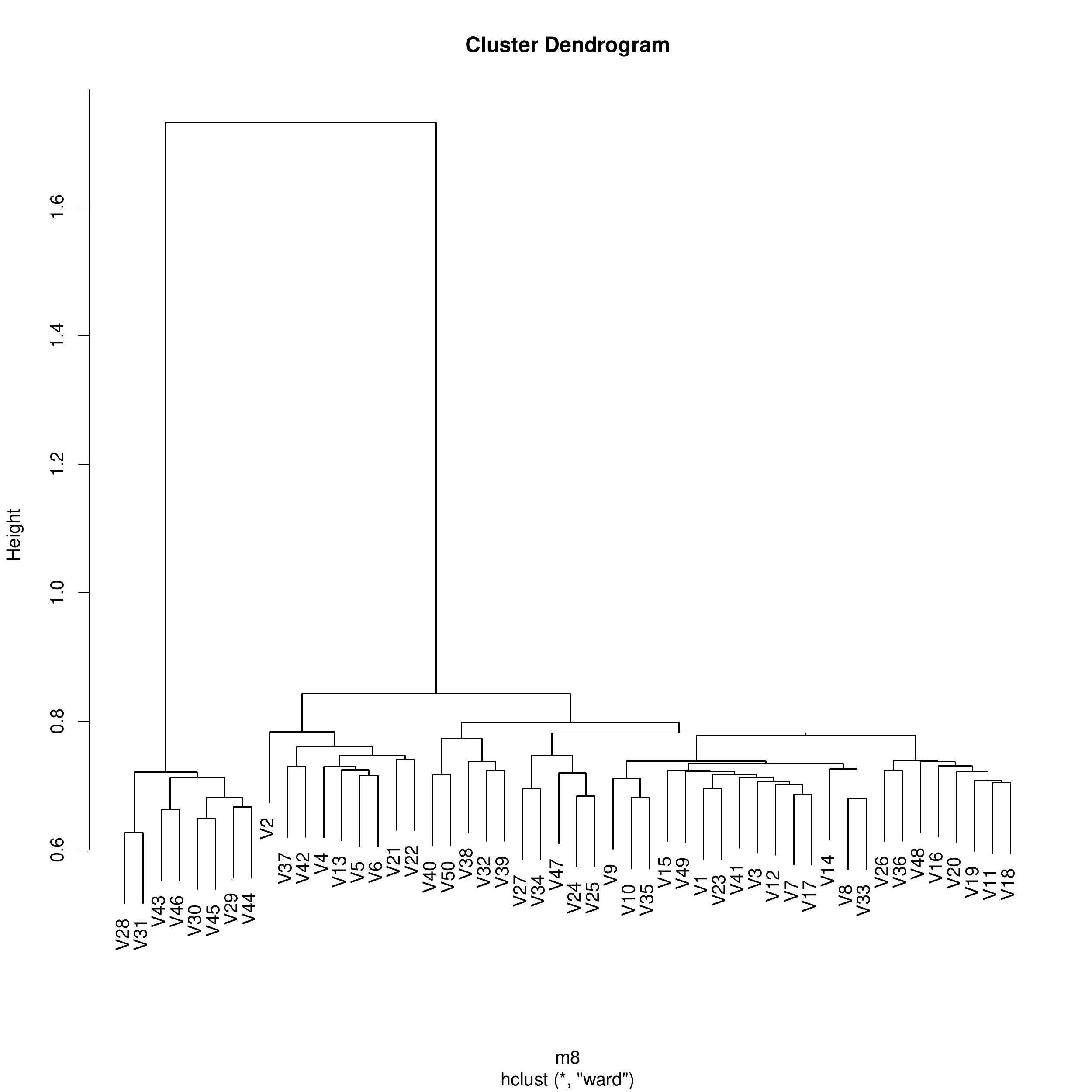}
  \hfill
  \includegraphics[width=.6\linewidth, trim=30 70 0 70, clip, angle=-90]{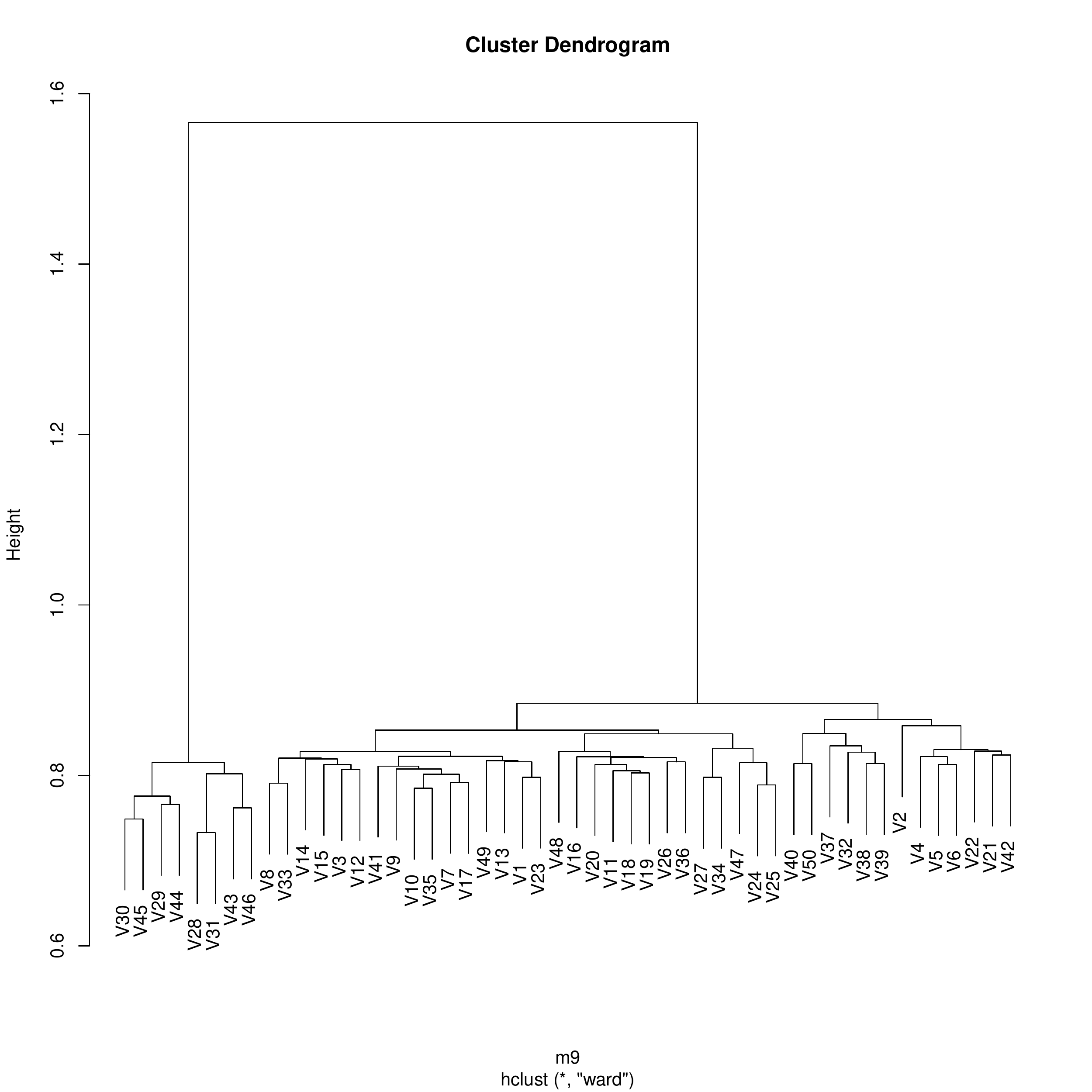}
  \hspace*{-2em}

  \medskip%
  \hspace*{-2em}
  \includegraphics[width=.6\linewidth, trim=30 70 0 50, clip, angle=-90]{experiment_1/plots/dendrograms/ward/ward_95}
  \caption{%
    \label{fi:dendro-1-ward.3}
    Dendrograms for Experiment 1, using Ward clustering, for discounting
    factors $.8$ (top left), $.9$ (top right) and $.95$ (bottom),
    respectively.  Fake papers are numbered 28-31 (Antkare) and 43-46
    (SCIgen), the others are genuine.}
\end{figure}

\begin{figure}[tbp]
  \centering
  \includegraphics[width=.8\linewidth, trim=0 0 0 0, clip]{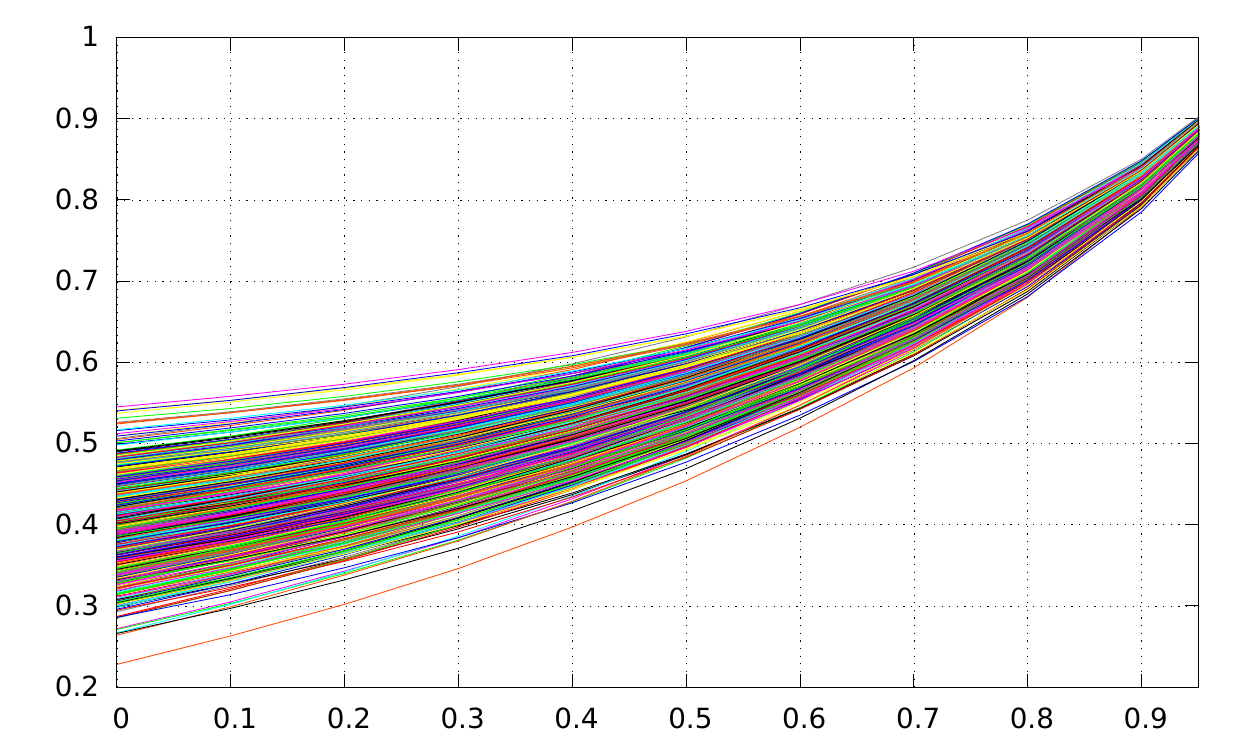}

  \includegraphics[width=.8\linewidth, trim=0 0 0 0, clip]{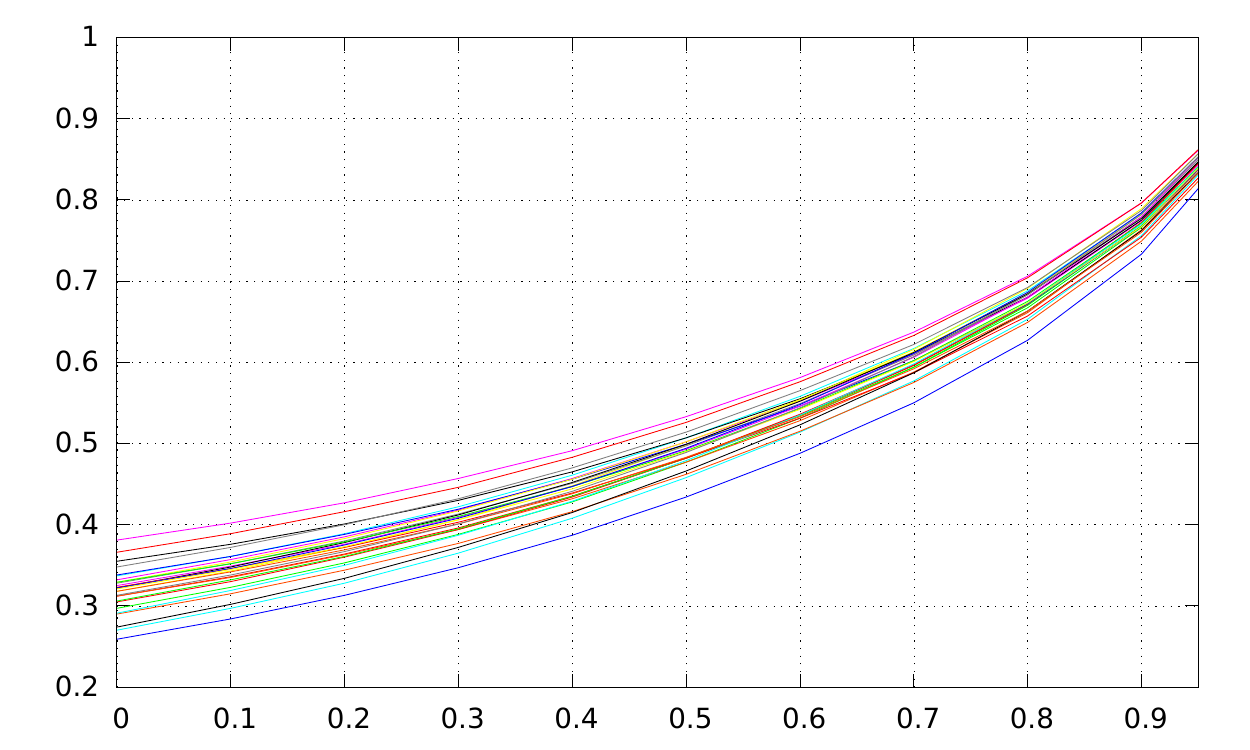}

  \includegraphics[width=.8\linewidth, trim=0 0 0 0, clip]{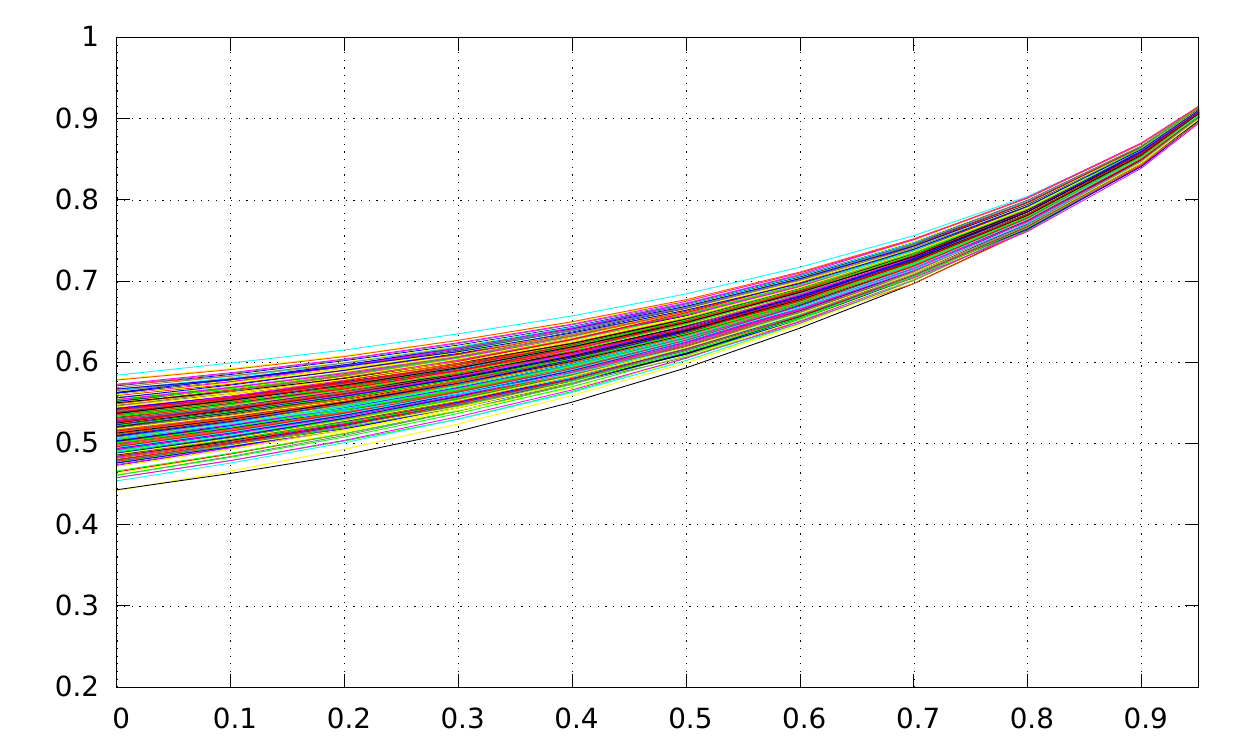}
  \caption{%
    \label{fi:disc-vs-dist-1}
    Distances between all pairs of genuine papers (top), all pairs of
    fake papers (middle), and between genuine and fake papers (bottom)
    in Experiment~1, depending on the discounting factor.  Lines between
    points have been added for visualization.}
\end{figure}

\begin{figure}
  \includegraphics[width=1.2\linewidth, trim=30 70 0 70, clip, angle=-90]{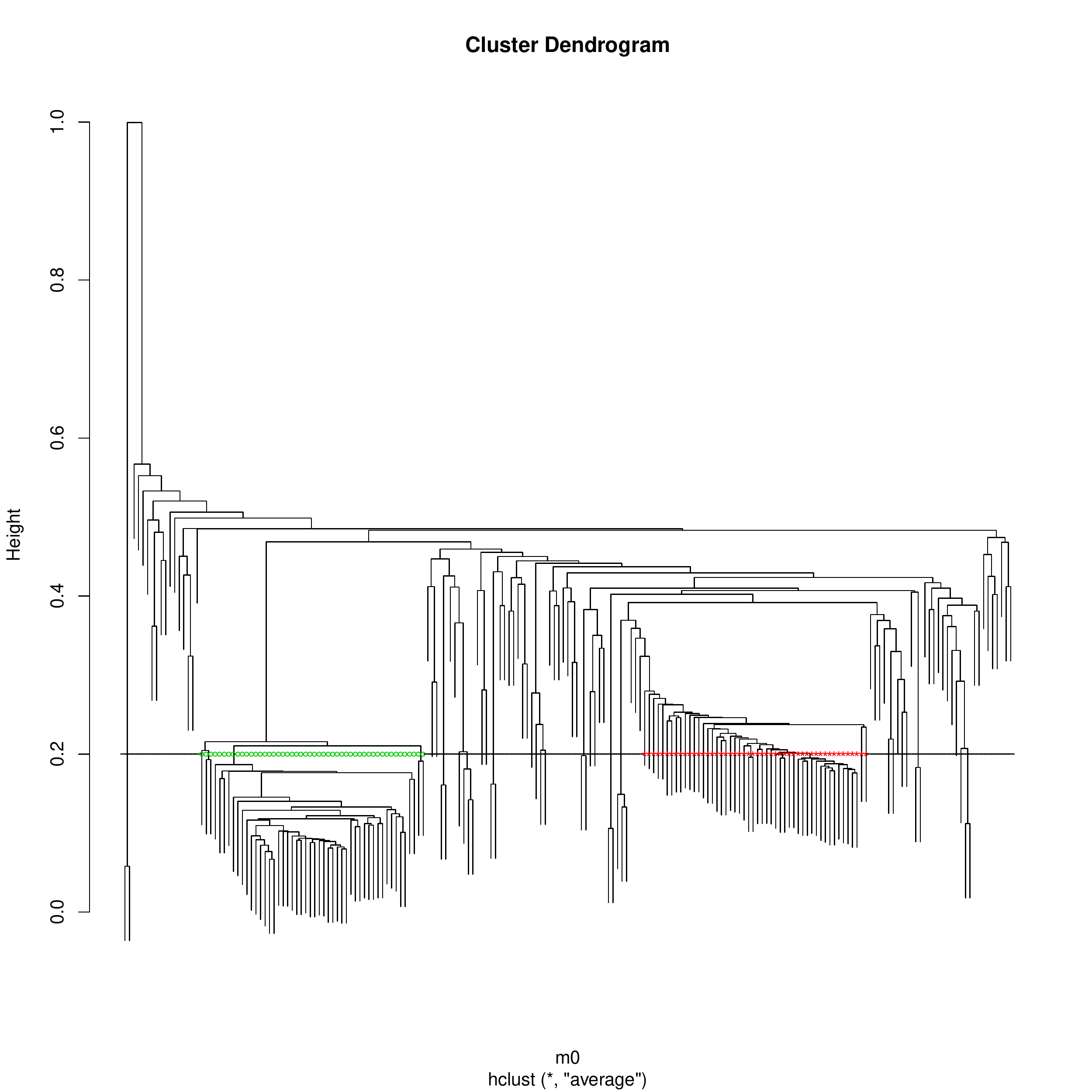}
  \caption{%
    \label{fi:dendro-2-avg.0}
    Dendrogram for Experiment 2, using average clustering, for
    discounting factor $0$.  Black dots mark \texttt{arxiv} papers,
    green marks SCIgen papers, and \texttt{automogensen} papers are
    marked red.}
\end{figure}

\begin{figure}
  \includegraphics[width=1.2\linewidth, trim=30 70 0 70, clip, angle=-90]{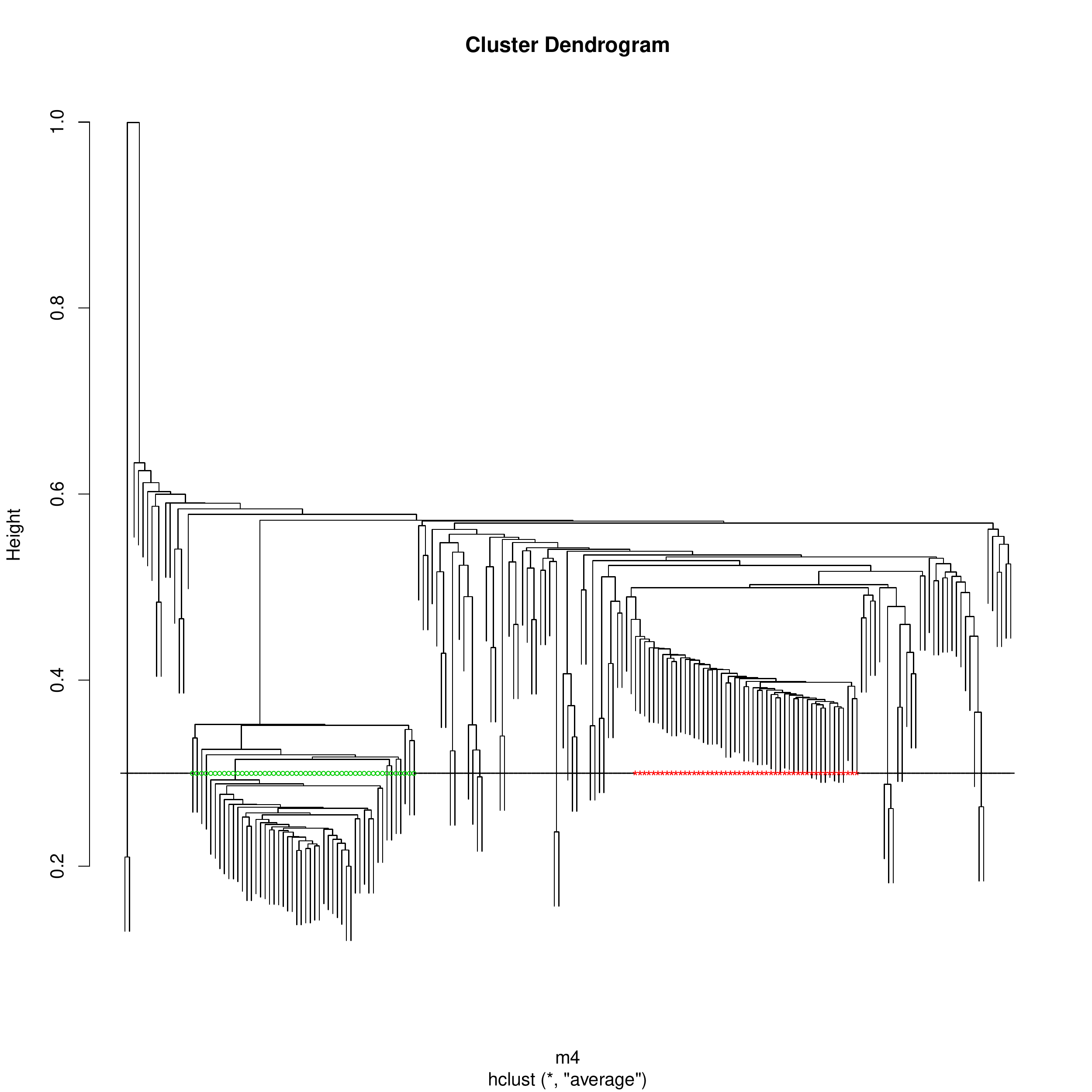}
  \caption{%
    \label{fi:dendro-2-avg.4}
    Dendrogram for Experiment 2, using average clustering, for
    discounting factor $.4$.  Black dots mark \texttt{arxiv} papers,
    green marks SCIgen papers, and \texttt{automogensen} papers are
    marked red.}
\end{figure}

\begin{figure}
  \includegraphics[width=1.2\linewidth, trim=30 70 0 70, clip, angle=-90]{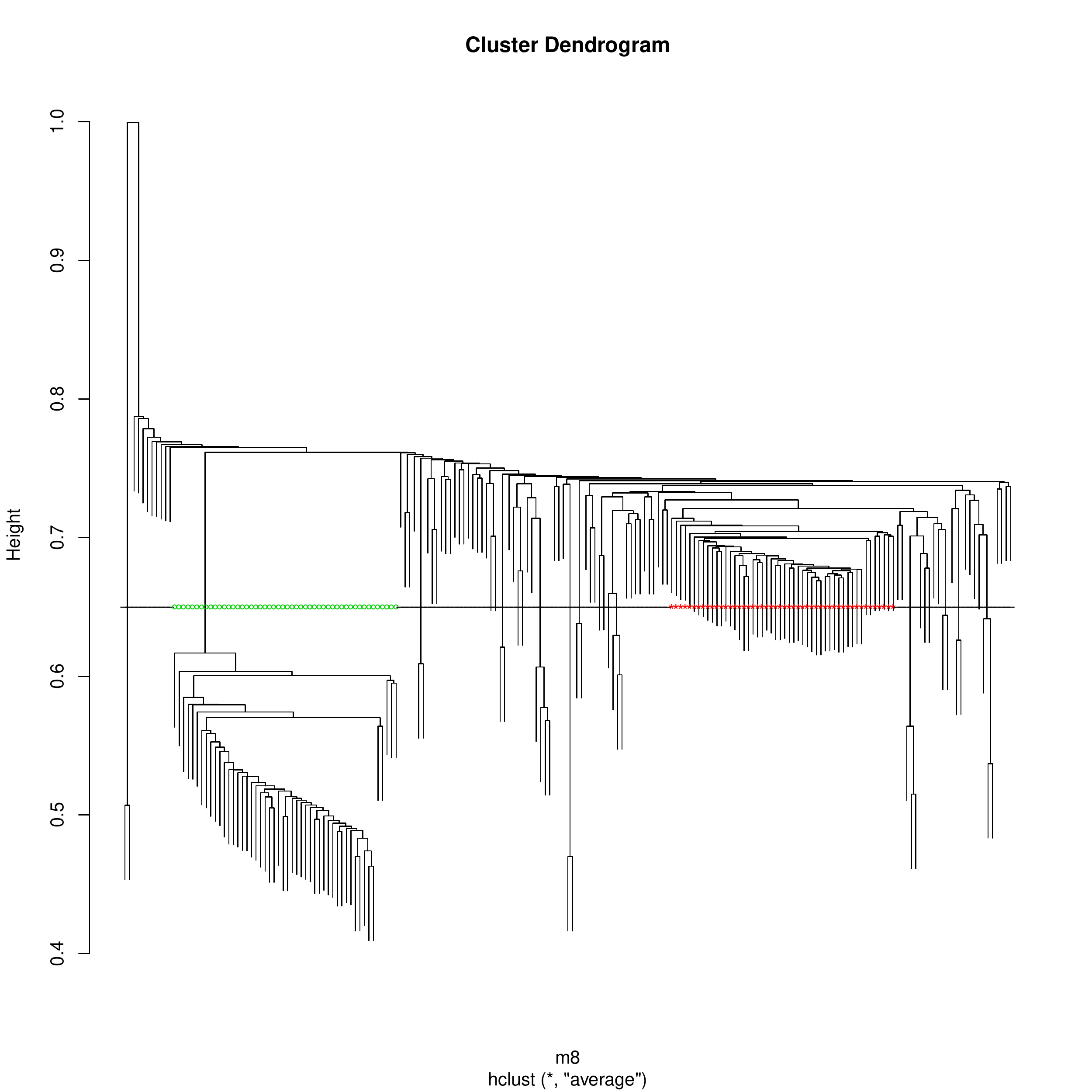}
  \caption{%
    \label{fi:dendro-2-avg.8}
    Dendrogram for Experiment 2, using average clustering, for
    discounting factor $.8$.  Black dots mark \texttt{arxiv} papers,
    green marks SCIgen papers, and \texttt{automogensen} papers are
    marked red.}
\end{figure}

\begin{figure}
  \includegraphics[width=1.2\linewidth, trim=30 70 0 70, clip, angle=-90]{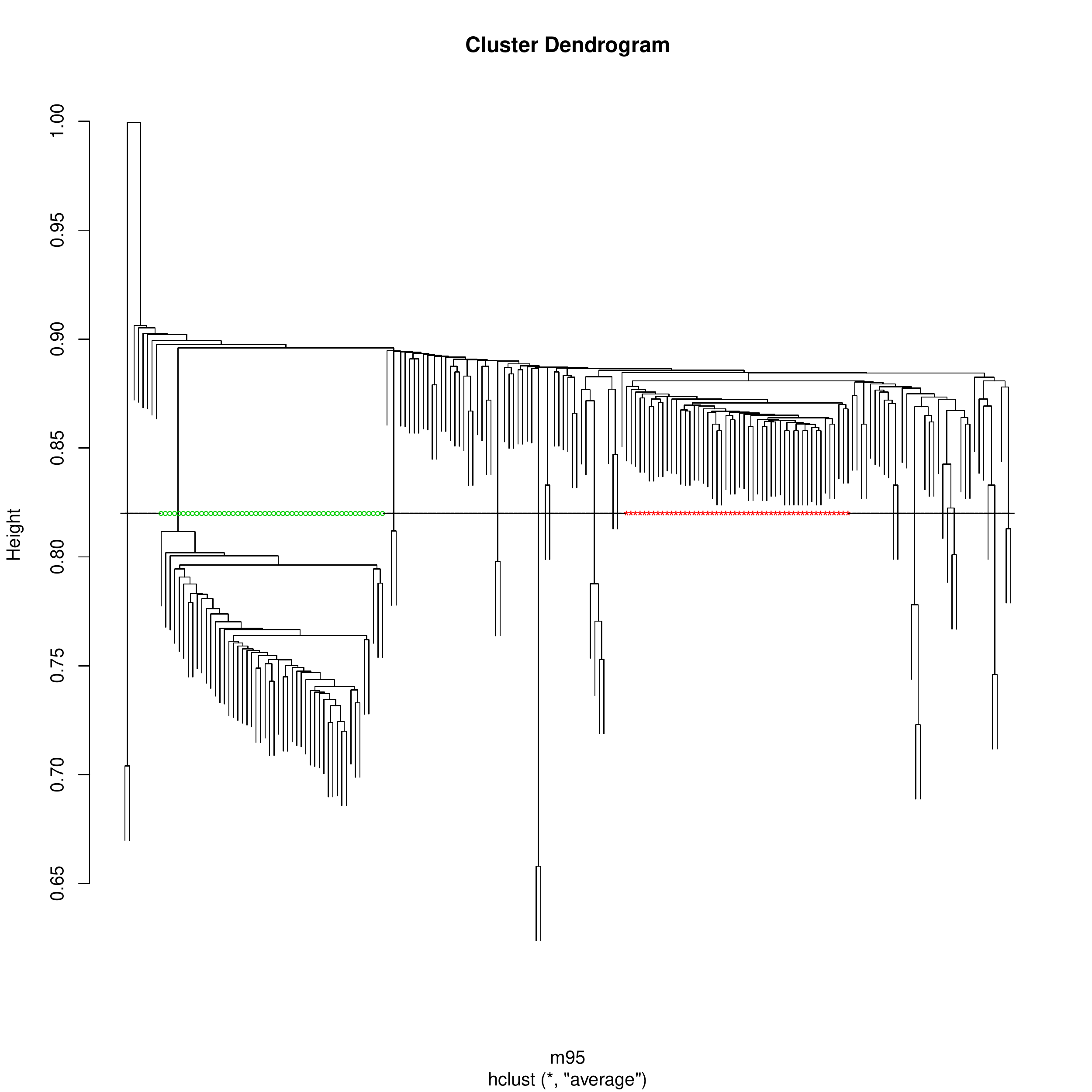}
  \caption{%
    \label{fi:dendro-2-avg.95}
    Dendrogram for Experiment 2, using average clustering, for
    discounting factor $.95$.  Black dots mark \texttt{arxiv} papers,
    green marks SCIgen papers, and \texttt{automogensen} papers are
    marked red.}
\end{figure}

\begin{figure}
  \includegraphics[width=1.2\linewidth, trim=30 70 0 70, clip, angle=-90]{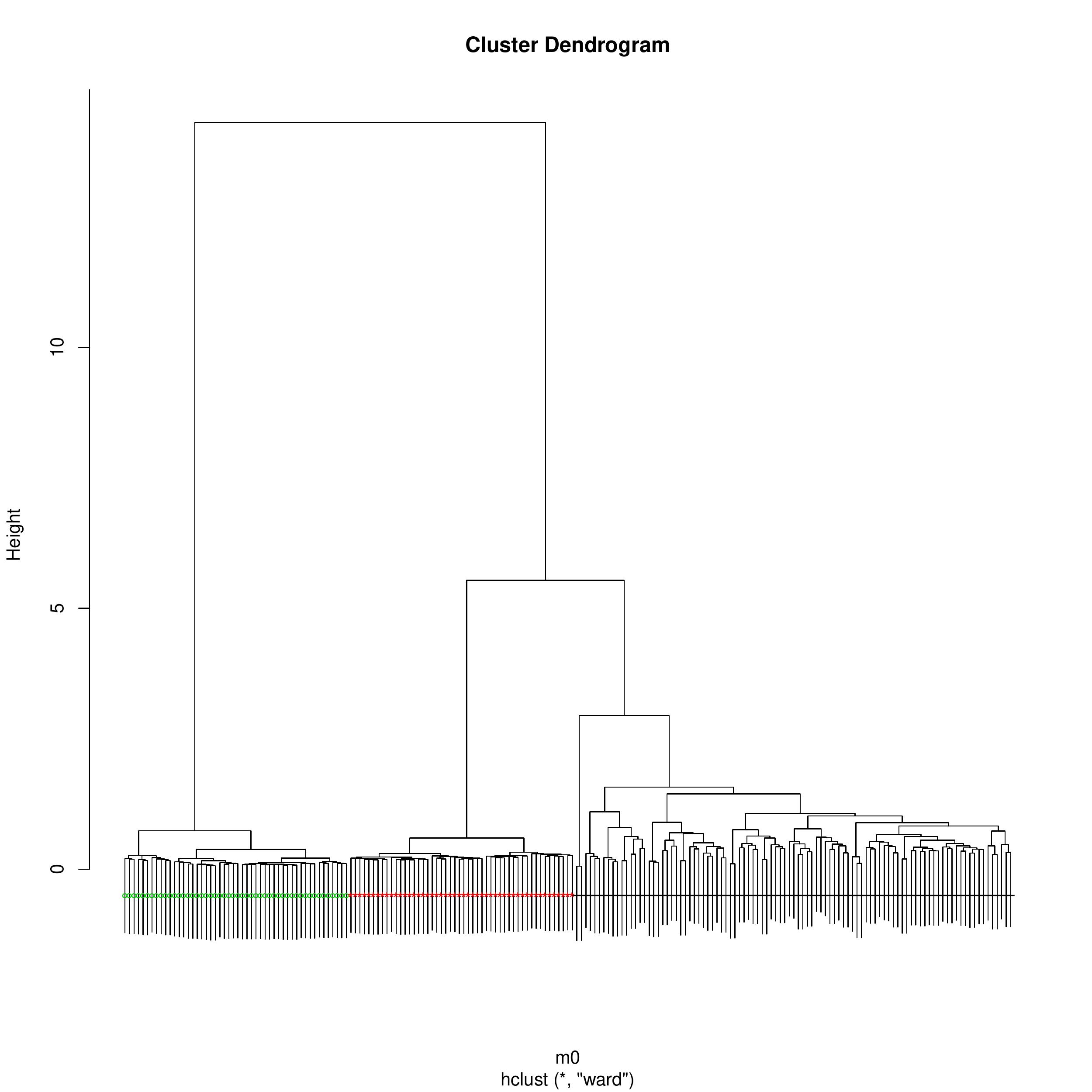}
  \caption{%
    \label{fi:dendro-2-ward.0}
    Dendrogram for Experiment 2, using Ward clustering, for discounting
    factor $0$.  Black dots mark \texttt{arxiv} papers, green marks
    SCIgen papers, and \texttt{automogensen} papers are marked red.}
\end{figure}

\begin{figure}
  \includegraphics[width=1.2\linewidth, trim=30 70 0 70, clip, angle=-90]{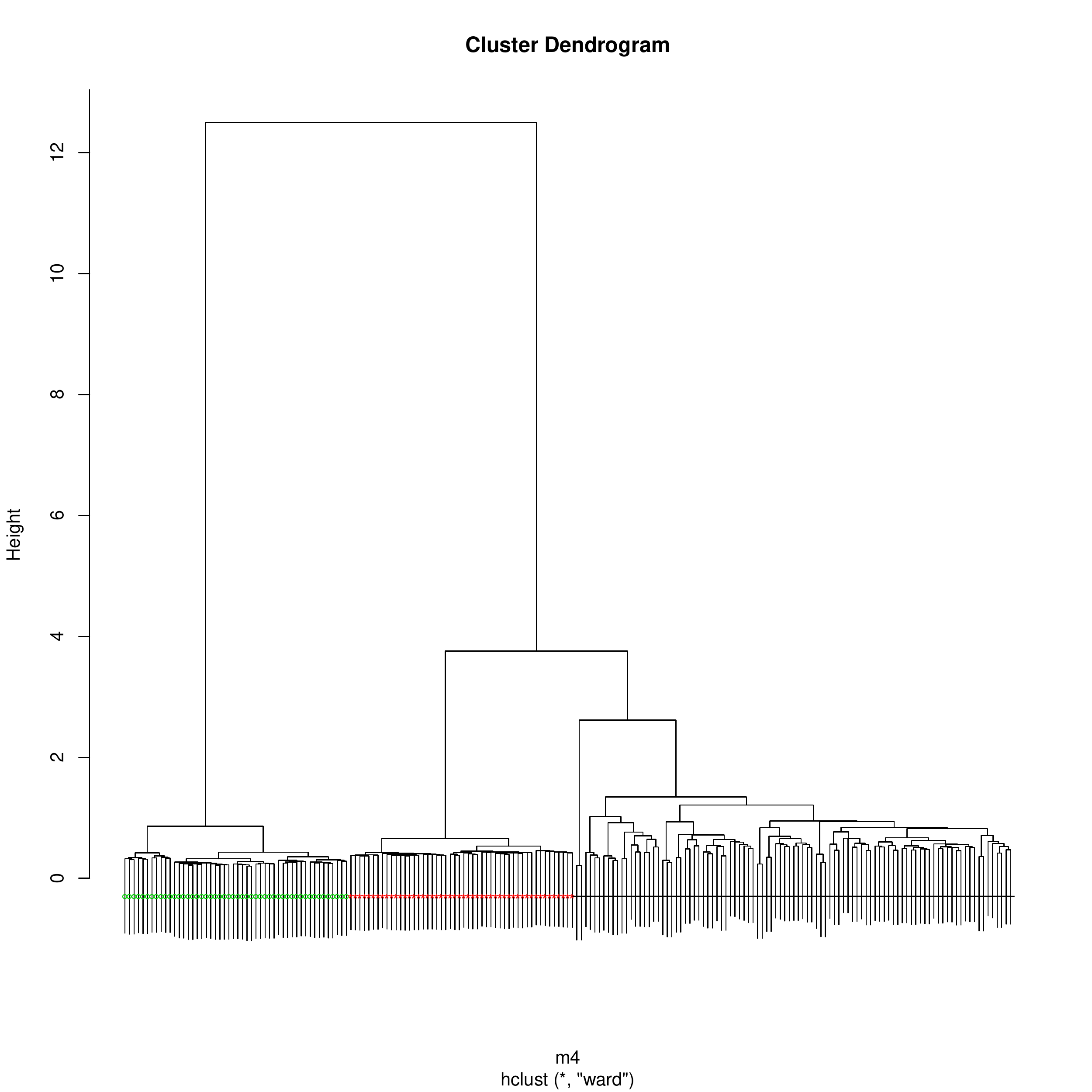}
  \caption{%
    \label{fi:dendro-2-ward.4}
    Dendrogram for Experiment 2, using Ward clustering, for discounting
    factor $.4$.  Black dots mark \texttt{arxiv} papers, green marks
    SCIgen papers, and \texttt{automogensen} papers are marked red.}
\end{figure}

\begin{figure}
  \includegraphics[width=1.2\linewidth, trim=30 70 0 70, clip, angle=-90]{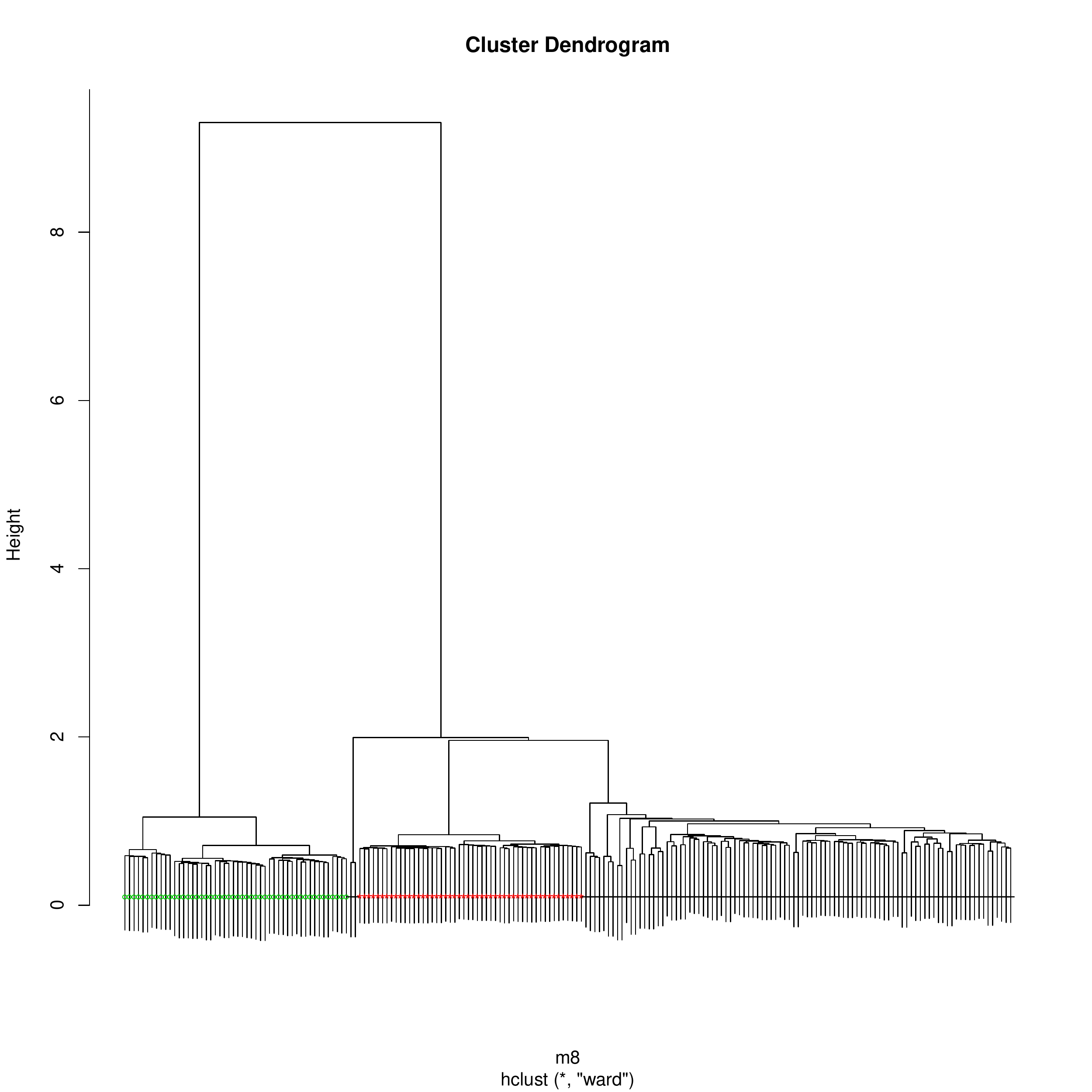}
  \caption{%
    \label{fi:dendro-2-ward.8}
    Dendrogram for Experiment 2, using Ward clustering, for discounting
    factor $.8$.  Black dots mark \texttt{arxiv} papers, green marks
    SCIgen papers, and \texttt{automogensen} papers are marked red.}
\end{figure}

\end{document}